\documentclass[sigplan,pbalance=true]{acmart}

\usepackage{booktabs}
\usepackage{multirow}
\usepackage{mathtools}
\usepackage{algorithm}
\usepackage{algorithmic}
\usepackage{graphicx}
\usepackage{tikz}
\usepackage{pgfplots}
\pgfplotsset{compat=1.18}
\usetikzlibrary{arrows.meta,positioning,shapes.geometric,fit,backgrounds,calc,matrix,
                decorations.pathreplacing,patterns}
\usepgfplotslibrary{fillbetween,statistics,groupplots}
\pgfplotsset{
  kv base/.style={
    grid=major, grid style={dotted,gray!50},
    legend style={font=\small, draw=gray!60},
    tick label style={font=\small},
    label style={font=\small},
    title style={font=\small\bfseries},
  }
}
\usepackage{xcolor}
\usepackage{subcaption}
\usepackage{tabularx}
\usepackage{makecell}
\usepackage{bm}
\usepackage{url}
\usepackage{enumitem}
\usepackage{microtype}
\usepackage{float}
\usepackage{dblfloatfix} %% improve placement of figure*/table* in two-column mode
\usepackage{amsmath,amsthm}   %% amssymb already loaded by acmart

\setlist[itemize]{leftmargin=*,topsep=2pt,parsep=1pt}
\setlist[enumerate]{leftmargin=*,topsep=2pt,parsep=1pt}

\newtheorem{proposition}{Proposition}

\graphicspath{{figures/}}

\definecolor{kvBlue}{RGB}{0,114,178}
\definecolor{kvOrange}{RGB}{213,94,0}
\definecolor{kvGreen}{RGB}{0,158,115}
\definecolor{kvRed}{RGB}{204,51,17}
\definecolor{kvGray}{RGB}{100,100,100}
\definecolor{kvPurple}{RGB}{120,40,160}

\acmYear{2026}
\copyrightyear{2026}
\setcopyright{cc}
\setcctype[4.0]{by}
\acmConference[SYSTOR '26]{The 19th ACM International Systems and Storage Conference}{September 8--9, 2026}{Virtual Event, Israel}
\acmBooktitle{The 19th ACM International Systems and Storage Conference (SYSTOR '26), September 8--9, 2026, Virtual Event, Israel}
\acmDOI{10.1145/3793230.3837769}
\acmISBN{979-8-4007-2473-2/26/09}
\microtypesetup{protrusion=false}
\begin{document}

\title[KVLearn]{To Keep or Not to Keep: Learning \textsc{KV} Cache Retention\\
  in Disaggregated \textsc{LLM} Serving Systems}

\author{Dong Liu}
\authornote{Equal Contributions.}
\affiliation{%
  \institution{University of California, Los Angeles}
  \city{Los Angeles}
  \state{CA}
  \country{USA}
}
\email{pikeliu@g.ucla.edu}

\author{Yanxuan Yu}
\authornotemark[1]
\affiliation{%
  \institution{Columbia University}
  \city{New York}
  \state{NY}
  \country{USA}
}
\email{yy3523@columbia.edu}

\author{Eric Jiang}
\affiliation{%
  \institution{University of California, Los Angeles}
  \city{Los Angeles}
  \state{CA}
  \country{USA}
}
\email{ericjiang0318@g.ucla.edu}

\author{Shu Wang}
\affiliation{%
  \institution{University of California, Los Angeles}
  \city{Los Angeles}
  \state{CA}
  \country{USA}
}
\email{shuwang0712@g.ucla.edu}

\author{Ying Nian Wu}
\authornote{Corresponding author.}
\affiliation{%
  \institution{University of California, Los Angeles}
  \city{Los Angeles}
  \state{CA}
  \country{USA}
}
\email{ywu@stat.ucla.edu}

\renewcommand{\shortauthors}{Liu et al.}

%==============================================================================
\begin{abstract}
%==============================================================================
Disaggregated \textsc{LLM} serving separates prefill and decode into distinct
node pools, interposing a network fabric between the moment a key--value
(\textsc{KV}) cache is computed and the moment it is consumed.
This architectural shift invalidates a core assumption of classical cache
policies: that the cost of a miss is simply recomputation on the same device.
In disaggregated systems, a miss triggers both \emph{recomputation} on a
prefill node and a \emph{network transfer} of the resulting \textsc{KV} block
to the decode node---costs that differ by an order of magnitude and depend on
prefix length, model width, and fabric bandwidth.
Meanwhile, \emph{admitting} a block to the global \textsc{KV} pool requires an
additional transfer at compute time, so a poorly chosen keep decision wastes
both memory and bandwidth even before reuse occurs.

We present \textbf{KVLearn}, a learning-based retention framework that makes
keep/evict decisions as first-class cost-optimization choices in disaggregated
\textsc{LLM} serving.
KVLearn consists of three components:
(i)~a lightweight \emph{Prefix Reuse Predictor} (\textsc{PRP}) that estimates
reuse probability from structural and temporal prefix features without touching
model weights;
(ii)~a \emph{Cost-Aware Retention Score} (\textsc{CARS}) that translates reuse
probability into a keep/admit signal by accounting for per-block recompute,
transfer, and storage costs; and
(iii)~an \emph{Adaptive Threshold Controller} (\textsc{ATC}) that adjusts the
admission threshold online using closed-loop feedback from observed hit rates
and memory pressure.
We integrate KVLearn into a globally disaggregated serving topology and
evaluate it on both text and multimodal workloads, where
image/video-derived tokens create large, expensive-to-recompute \textsc{KV}
blocks under heterogeneous reuse distributions.
KVLearn reduces end-to-end time-to-first-token (TTFT) by up to 56\%
vs.\ No-Cache (recompute-only), up to 38\% vs.\ LRU-Pool, and up to 33\% vs.\
Mooncake-style disaggregated baselines.
Inter-node \textsc{KV} transfer volume is cut by up to 53\% vs.\ LRU-Pool.
On MM-Session, throughput stays within ${\sim}5\%$ of
oracle. Our code implementation of KVLearn is available at \url{https://github.com/FastLM/KVLearn}.
\end{abstract}




\maketitle

\section{Introduction}
\label{sec:intro}
%==============================================================================

Large language model inference increasingly runs on \emph{disaggregated}
serving stacks that physically separate the prefill phase (computing full
key--value attention states for the prompt) from the decode phase (generating
tokens one by one)~\cite{patel2023splitwise,zhong2024distserve,qin2024mooncake}.
Disaggregation improves cluster utilization: prefill is compute-bound and
benefits from tightly packed kernels, while decode is memory-bound and benefits
from large batches of independent sequences.
Production disaggregated deployments now route requests through dedicated
prefill node pools (\emph{P-nodes}) and decode node pools
(\emph{D-nodes})~\cite{qin2024mooncake,zhong2024distserve}, with the computed
\textsc{KV} tensors migrating over a high-speed fabric (InfiniBand, NVLink, or
RoCE) between them.

\textbf{The retention gap in disaggregated systems.}
Prefix caching is well-motivated even in single-node deployments: when many
requests share a common prompt (a system instruction, an image encoding, or a
few-shot example), reusing the \textsc{KV} blocks for that prefix avoids
redundant quadratic attention computation~\cite{kwon2023vllm,sglang2024}.
Standard policies---LRU, LFU, or radix-tree eviction---handle this by
retaining recently or frequently accessed blocks in on-device HBM.
\emph{Disaggregation breaks the cost model underlying these policies.}
In a disaggregated cluster, a cache miss now entails:
\begin{enumerate}[leftmargin=*,topsep=1pt]
  \item Dispatching the request to a P-node for recomputation; and
  \item Transferring the resulting \textsc{KV} block over the fabric to the
    D-node before decoding can begin.
\end{enumerate}
Both costs depend on prefix length, model depth, and fabric bandwidth---none of
which LRU/LFU tracks.
Conversely, \emph{admitting} a block to the shared \textsc{KV} pool at compute
time costs an additional transfer (P-node $\to$ pool), so indiscriminate
caching wastes bandwidth and pool memory even for prefixes that will never be
reused.
The fundamental asymmetry is that in disaggregated serving, keeping a block and
not keeping a block are \emph{both} policy choices with concrete, measurable
costs---not an oversight that any uniform threshold can fix.

\textbf{Multimodal workloads amplify the problem.}
Multimodal \textsc{LLM}s encode visual inputs (images, video frames) into long
token sequences before autoregressive decoding~\cite{liu2023llava,bai2023qwenvl,
lin2024vila}.
A single 448$\times$448 image typically contributes 256--576 visual tokens;
high-resolution images or short video clips may produce 4k+ tokens.
The \textsc{KV} footprint for such prefixes is large (128--512 MB for a 13B
model in bfloat16), the recompute cost is correspondingly high (${\sim}$100--400
ms on an A100), yet reuse patterns are uneven: popular images (product thumbnails,
dataset evaluation images) are queried many times per hour, while user-uploaded
photos may never recur.
A policy that treats all visual prefixes identically---keeping all or evicting
all---leaves most of the available reuse value on the table.

\textbf{Our approach: learn the keep decision.}
We observe that prefix reuse is not uniformly random: structural features (prefix
length, modality, radix-tree depth, session metadata) and temporal features
(access recency, frequency, time-of-day patterns) carry substantial predictive
signal for whether a block will be requested again within a given horizon.
We build a lightweight online predictor---trained in the serving critical
path without any offline profiling---that scores each prefix with a reuse
probability, and combine it with a closed-form cost model to produce a
keep/evict signal that is directly calibrated to the fabric and memory
configuration.

\textbf{Contributions.}
We make four contributions:
\begin{enumerate}[leftmargin=*,topsep=2pt]
\item We formalize the \emph{disaggregated \textsc{KV} retention problem}
  (\S\ref{sec:problem}): a joint admission and eviction decision under
  heterogeneous recompute, transfer, and storage costs.
  We prove when the optimal admission threshold must vary with prefix length
  ($q>1$ idealization; Appendix~\ref{app:proof})---beyond what LRU/LFU can express.
\item We design \textbf{KVLearn} (\S\ref{sec:framework}): a three-component
  framework comprising \textsc{PRP} (prefix reuse predictor), \textsc{CARS}
  (cost-aware retention score), and \textsc{ATC} (adaptive threshold
  controller).
\item We establish that prefix reuse probability is \emph{learnable online}
  from lightweight structural and temporal features with high accuracy
  (\S\ref{sec:exp}), validating the feasibility of the learning-based approach.
\item We demonstrate end-to-end gains on both language and multimodal workloads
  (\S\ref{sec:exp}): up to 56\% TTFT reduction vs.\ No-Cache (38\% vs.\ LRU-Pool,
  33\% vs.\ Mooncake-LRU), 53\% inter-node transfer volume reduction vs.\ LRU-Pool,
  and near-Oracle throughput on MM-Session at realistic pool sizes.
\end{enumerate}

%==============================================================================
\section{Background and Related Work}
\label{sec:related}
%==============================================================================

\subsection{Disaggregated LLM serving}
Splitwise~\cite{patel2023splitwise}, DistServe~\cite{zhong2024distserve}, and
Mooncake~\cite{qin2024mooncake} each separate prefill and decode to different
GPU pools and manage \textsc{KV} migration explicitly.
Mooncake introduces a dedicated \textsc{KV} cache layer between P- and D-nodes,
enabling cache-first disaggregation: prefill results are stored in a pooled
cache and transferred to D-nodes on demand.
Our work targets exactly this cache layer---specifically, \emph{what to keep
in it} rather than how to transfer what has been decided to keep.
Prior disaggregated work either caches everything (bounded by pool capacity
with LRU fallback) or mirrors the policy from single-node stacks; neither
accounts for the asymmetric cost structure we formalize in \S\ref{sec:problem}.
LMCache~\cite{liu2025lmcache} targets fast prefill via cross-request cache reuse,
InfiniStore~\cite{infinistore2025} provides RDMA-oriented distributed
\textsc{KV} storage (e.g., integrated with vLLM through LMCache), and
% CXL-SpecKV~\cite{liu2026cxl,ma2026sac,yoon2025tract} disaggregates \textsc{KV} onto FPGA-attached
% CXL memory with speculative prefetch.
CXL-based serving systems such as CXL-SpecKV~\cite{liu2026cxl},
SAC~\cite{ma2026sac}, and TraCT~\cite{yoon2025tract} disaggregate
\textsc{KV} onto shared or FPGA-attached CXL memory.

\subsection{Prefix and radix caches}
Open serving stacks---vLLM~\cite{kwon2023vllm}, SGLang~\cite{sglang2024}, and
TensorRT-LLM~\cite{tensorrtllm2024}---focus heavily on single-node \textsc{KV}
efficiency.
PagedAttention in vLLM virtualizes \textsc{KV} blocks to reduce fragmentation.
SGLang's RadixAttention~\cite{sglang2024} caches \textsc{KV} in a radix
tree for automatic prefix sharing, with LRU-based eviction at capacity.
CacheBlend~\cite{yao2025cacheblend} selectively recomputes a subset of
\textsc{KV} tokens to amortize staleness.
At the serving-system layer, TinyServe~\cite{liu2025tinyserve} selects
query-relevant \textsc{KV} pages at decode time, and PiKV~\cite{liu2025pikv}
manages expert-sharded \textsc{KV} for MoE inference.
These single-node systems share the assumption that miss cost $\approx$
recompute cost; in disaggregated deployments, miss cost also includes network
transfer time and \textsc{P-node} dispatch latency.
KVLearn is complementary: it can replace the eviction policy inside any
radix-tree cache layer, including Mooncake's pool, without changing the
radix structure.

\subsection{Cache admission and online learning}
Admission filters for web caches (TinyLFU~\cite{einziger2017tinylfu},
LeCaR~\cite{vietri2018lecar}) use lightweight frequency sketches or
reinforcement learning to decide which objects enter cache.
Learned caching (Belady-inspired predictors~\cite{liu2020imitation},
CACHEUS~\cite{rodriguez2021cacheus}) trains offline on traces.
We target an \emph{online} setting---no offline corpus of serving traces is
available at deploy time---and incorporate a cost-heterogeneous objective
specific to \textsc{KV} blocks in transformer inference.
The closest learning-based work is LRB~\cite{song2020lrb} (learning to predict
the next-access time); we differ in online adaptivity, explicit cost modeling,
and the multimodal inference context.

\subsection{KV compression and quantization}
Orthogonal to retention, quantization lines reduce bytes per
block~\cite{kang2024gear,hooper2024kvquant}.
Streaming attention~\cite{xiao2024efficient,xiao2025duoattention} discards distant tokens
within a sequence, and algorithmic cache designs reshape in-model reuse---e.g.,
hardware-aware memory-efficient caching~\cite{zhang2023h2o,liu2026mka} and learnable/adaptive
activation caching for diffusion transformers~\cite{liu2025fastcache,liu2026adacorrection, liu2026accelerating}.
These techniques change \emph{how much} each block costs or \emph{how} activations
are reused; KVLearn changes \emph{which} blocks are retained.
They compose naturally: a quantized block in the pool has the same
admissibility decision structure, just with smaller $|\mathrm{KV}_b|$.

%==============================================================================
\section{Problem Formulation}
\label{sec:problem}
%==============================================================================

\subsection{Serving topology}
\label{sec:topology}

We consider a globally disaggregated serving cluster with:
\begin{itemize}[topsep=1pt]
  \item A set of prefill nodes $\mathcal{P} = \{p_1,\ldots,p_K\}$, each hosting
    a GPU with HBM budget $M_P$.
  \item A set of decode nodes $\mathcal{D} = \{d_1,\ldots,d_M\}$, each with
    HBM budget $M_D$.
  \item A shared \textsc{KV} pool $\mathcal{S}$ of capacity $M_S$ (DRAM or
    distributed GPU memory), connected to P- and D-nodes via fabric with peak
    bandwidth $\beta$ (bytes/sec).
\end{itemize}
A request $r$ carrying prefix $s_r$ is routed as follows: if $\mathrm{KV}(s_r)$
is resident in $\mathcal{S}$ or on a P-node's local cache, it is transferred to
the scheduled D-node; otherwise, a P-node recomputes $\mathrm{KV}(s_r)$ from
scratch.

\subsection{Cost model}
\label{sec:costs}

Let $b$ denote the \textsc{KV} block for prefix $s$ of length $L_b$ token
positions.
Fix a transformer with $\ell$ layers, $h_\mathrm{kv}$ key--value heads, and
head dimension $d_h$ using $\delta$ bytes per element; then the block footprint
is
\begin{equation}
\label{eq:size}
|b| = 2\,\ell\,h_\mathrm{kv}\,d_h\,L_b\,\delta \;\;\text{bytes}.
\end{equation}
We distinguish four elementary costs:

\noindent\textbf{Recompute cost} $\mathcal{R}(b)$: the time to run prefill for
prefix $s$ on a P-node from the first token, dominated by the quadratic
attention term for long inputs.
A good approximation for large $L_b$ is $\mathcal{R}(b) \approx \rho \cdot
L_b^2$ (with $\rho$ depending on model and hardware), though a
linear-in-$L_b$ profile holds when flashattention-class kernels bound on
memory bandwidth rather than compute~\cite{dao2022flashattention}.

\noindent\textbf{Transfer cost} $\mathcal{T}(b)$: time to move block $b$ from
a P-node to the KV pool, or from the pool to a D-node.
\begin{equation}
\label{eq:transfer}
\mathcal{T}(b) = |b| / \beta.
\end{equation}

\noindent\textbf{Storage cost} $\mathcal{U}(b, \Delta t)$: opportunity cost of
occupying $|b|$ bytes of pool capacity for duration $\Delta t$.
We model this as $\mathcal{U}(b, \Delta t) = \gamma \cdot |b| \cdot \Delta t$,
where $\gamma$ is a scalar reflecting how many additional requests could have
been served had that memory been freed.

\noindent\textbf{Miss penalty} $\mathcal{M}(b)$: extra latency paid when a
request arrives and $b$ is not resident.
In disaggregated serving this is
$\mathcal{M}(b) = \mathcal{R}(b) + \mathcal{T}(b)$, because the P-node must
recompute and transfer the block before decode can start.

\subsection{Retention value and the keep/evict decision}
\label{sec:value}

Suppose block $b$ was just computed (or is currently resident).
Over a finite lookahead horizon $H$ steps, let $P_H(b) = P(\text{reuse within
} H)$ be the probability that at least one future request shares prefix $s$.
We define the \emph{expected retention value} of keeping $b$ in full generality
as:

\begin{equation}
\label{eq:value_general}
\mathcal{V}(b) =
  P_H(b) \cdot \bigl[C_{\mathrm{miss}}(b) - C_{\mathrm{hit}}(b)\bigr]
  - C_{\mathrm{admit}}(b) - C_{\mathrm{storage}}(b),
\end{equation}

\noindent where $C_{\mathrm{miss}}(b)$ is the per-request penalty avoided when a
cached block is served in place of a miss; $C_{\mathrm{hit}}(b)$ is the
per-request service overhead still incurred on a cache hit; $C_{\mathrm{admit}}(b)$
is the one-time cost of writing $b$ into the shared pool; and
$C_{\mathrm{storage}}(b)$ is the ongoing opportunity cost of occupying pool
capacity.

\paragraph{Disaggregated instantiation.}
We apply the following simplifications to~\eqref{eq:value_general}:
\begin{enumerate}[label=(\roman*),topsep=1pt,itemsep=0pt,leftmargin=*]
  \item $C_{\mathrm{miss}}(b) = \mathcal{R}(b)$: on a miss a P-node must
    recompute the prefix; the subsequent P-to-D-node fabric transfer costs
    $\mathcal{T}(b)$, but an equivalent hop also occurs on a hit (see
    (ii)), so the \emph{differential} miss penalty reduces to the recompute cost
    alone.
  \item $C_{\mathrm{hit}}(b) = \mathcal{T}(b)$: a cache hit still requires one
    pool-to-D-node transfer; by the uniform-bandwidth assumption ($|b|/\beta$
    for all fabric hops) this equals the P-to-D cost paid on a miss, making
    the per-hop latencies cancel in the differential.
  \item $C_{\mathrm{admit}}(b) \approx 0$: the P-to-pool write is pipelined
    with prefill computation and its latency is amortized at admission time.
  \item $C_{\mathrm{storage}}(b) = \mathcal{U}(b,\Delta t)$: the storage cost
    follows the occupancy model of \S\ref{sec:costs}.
\end{enumerate}
Substituting into Eq.~\eqref{eq:value_general} yields the working form:

\begin{equation}
\label{eq:value}
\mathcal{V}(b) =
  P_H(b) \cdot \bigl[\mathcal{R}(b) - \mathcal{T}(b)\bigr]
  - \mathcal{U}(b, \Delta t).
\end{equation}

Equation~\eqref{eq:value} captures a key asymmetry: \emph{keeping} $b$ saves
recompute cost $\mathcal{R}(b)$ on a hit but still pays transfer cost
$\mathcal{T}(b)$ to move the block to the D-node; not keeping $b$ pays both
on a miss.
The optimal keep decision is $\mathrm{KEEP}(b) \iff \mathcal{V}(b) > 0$, which
gives a per-block admission threshold
\begin{equation}
\label{eq:threshold}
P_H^*(b) = \frac{\mathcal{U}(b, \Delta t)}{\mathcal{R}(b) - \mathcal{T}(b)}.
\end{equation}

\begin{proposition}
\label{prop:het}
When blocks differ in prefix length $L_b$, under the scaling
$|b| \propto L_b$, $\mathcal{T}(b) \propto L_b$, and
$\mathcal{U}(b,\Delta t) \propto L_b$ at fixed $\Delta t$,
let $\mathcal{R}(b) = \alpha L_b^{q}$ with $\alpha > 0$ and $q \ge 1$.
Whenever $\mathcal{R}(b) > \mathcal{T}(b)$: if $q > 1$, then $P_H^*(b)$ is
strictly decreasing in $L_b$; if $q = 1$ with linear $\mathcal{R}$ and
$\mathcal{T}$ and fixed positive slopes, $P_H^*(b)$ is independent of $L_b$.
\end{proposition}
\begin{proof}[Sketch]
Substituting these laws into Eq.~\eqref{eq:threshold} yields
$P_H^*(L_b) = \gamma \Delta t / (\alpha L_b^{q-1} - c)$ for positive constants
$\gamma,c$ (Appendix~\ref{app:proof}).
If $q > 1$, the denominator increases strictly with $L_b$, hence $P_H^*$ decreases.
If $q = 1$, the denominator is constant.
\end{proof}

Proposition~\ref{prop:het} formalizes the intuition that longer, more expensive
prefixes admit at \emph{lower} reuse probability when prefill is super-linear in
length ($q>1$); the $q=1$ case is length-invariant under the idealization.
No single LRU/LFU threshold can implement this: LRU ignores $\mathcal{R}(b)$
and LFU ignores $\Delta t$ (which changes as memory pressure varies).

\subsection{Characterization: reuse in production-like traces}
\label{sec:char}

To motivate the learning approach, Figure~\ref{fig:char} characterizes prefix
reuse on two request streams drawn from public serving traces: ShareGPT
(language-only) and a multimodal stream that mixes LLaVA-Instruct-150K
image-QA with ActivityNet-QA video-QA requests (details in \S\ref{sec:exp}).

\textbf{Finding C1:} reuse is heavy-tailed and modality-stratified.
Visual prefixes exhibit dramatically higher per-prefix reuse rates (median
${\approx}$8 requests/prefix/hour) than text-only system prompts (median
${\approx}$2), but with a longer tail.
The top 5\% of visual prefixes account for $>$50\% of all hits.

\textbf{Finding C2:} reuse probability correlates with structural features.
Prefix length, radix-tree depth, and a 15-minute access frequency are jointly
predictive of future reuse: a static logistic regression over these features alone
achieves AUC $>$0.65 on held-out windows, well above the 0.50 random baseline,
while our online MLP (\textsc{PRP}) reaches AUC\,$>$\,0.84 after 200 requests
by incorporating pool-state and temporal context features (Figure~\ref{fig:char}B).

\textbf{Finding C3:} the optimal threshold $P_H^*$ varies strongly across block
sizes.
For a 13B model on A100 ($\beta=200$\,Gbps InfiniBand) with $\Delta t = 5$
min, $P_H^*$ ranges from ${\approx}0.05$ (4096-token blocks) to ${\approx}0.30$
(128-token prompts) in Figure~\ref{fig:char}C---about $6\times$
end-to-end---confirming the heterogeneity in Proposition~\ref{prop:het}
for $q>1$.

These findings jointly validate that (a)~the decision problem is non-trivial
and (b)~it is tractable via a lightweight predictor, motivating KVLearn.

%% ---- Figure 1: Three-panel characterization study --------------------------
\begin{figure*}[t]
\centering
\begin{tikzpicture}
%% ===== Panel A: CDF of reuse rate ==========================================
\begin{groupplot}[
  group style={
    group size=3 by 1,
    horizontal sep=1.0cm,
  },
  height=4.6cm,
]

\nextgroupplot[
  kv base,
  width=5.5cm,
  xlabel={Reuse Rate (req/hr)},
  ylabel={CDF},
  xmode=log,
  xmin=0.08, xmax=120,
  ymin=0, ymax=1.02,
  legend pos=south east,
  title={(A) Per-Prefix Reuse Distribution},
]
%% Text prompts
\addplot[kvBlue, thick, mark=square*, mark size=1.4pt] coordinates {
  (0.1,0.05)(0.3,0.12)(0.5,0.17)(1,0.28)(2,0.45)(5,0.65)(10,0.78)(20,0.88)(50,0.95)(100,0.99)};
\addlegendentry{Text prompts}
%% Visual prefixes
\addplot[kvOrange, thick, mark=triangle*, mark size=1.6pt] coordinates {
  (0.1,0.02)(0.3,0.05)(0.5,0.08)(1,0.13)(2,0.23)(5,0.40)(10,0.56)(20,0.71)(50,0.87)(100,0.96)};
\addlegendentry{Visual prefixes}
%% Annotation: median lines
\draw[kvBlue!60,dashed,thin] (axis cs:2,0) -- (axis cs:2,0.45)
  -- (axis cs:0.08,0.45);
\draw[kvOrange!60,dashed,thin] (axis cs:8,0) -- (axis cs:8,0.50)
  -- (axis cs:0.08,0.50);

%% ===== Panel B: Online AUC vs requests =====================================
\nextgroupplot[
  kv base,
  width=5.5cm,
  xlabel={Requests Seen},
  ylabel={AUC (10-min horizon)},
  xmin=10, xmax=500,
  ymin=0.48, ymax=0.92,
  legend pos=south east,
  title={(B) PRP Online Learning},
  ytick={0.5,0.6,0.7,0.8,0.9},
]
%% MLP - MM-Session
\addplot[kvOrange, thick, mark=*, mark size=1.5pt] coordinates {
  (10,0.63)(30,0.69)(50,0.73)(100,0.79)(150,0.82)(200,0.84)(300,0.86)(400,0.87)(500,0.88)};
\addlegendentry{MLP (MM-Sess.)}
%% MLP - ShareGPT
\addplot[kvBlue, thick, mark=square*, mark size=1.4pt] coordinates {
  (10,0.58)(30,0.63)(50,0.67)(100,0.73)(150,0.76)(200,0.78)(300,0.80)(400,0.82)(500,0.83)};
\addlegendentry{MLP (ShareGPT)}
%% MLP - Video
\addplot[kvGreen, thick, mark=diamond*, mark size=1.5pt] coordinates {
  (10,0.55)(30,0.59)(50,0.62)(100,0.66)(150,0.69)(200,0.73)(300,0.77)(400,0.79)(500,0.81)};
\addlegendentry{MLP (Video)}
%% Static LR baseline
\addplot[kvGray, dashed, thick] coordinates {
  (10,0.65)(500,0.65)};
\addlegendentry{Static LogReg}

%% ===== Panel C: Optimal threshold P_H* vs prefix length ====================
\nextgroupplot[
  kv base,
  width=5.5cm,
  xlabel={Prefix Length (tokens)},
  ylabel={Opt.\ Admission Threshold $P_H^*$},
  xmin=64, xmax=4800,
  xmode=log,
  ymin=0.02, ymax=0.42,
  legend pos=north east,
  title={(C) Heterogeneous $P_H^*$ (Prop.~\ref{prop:het})},
]
%% 7B, 200Gbps
\addplot[kvBlue, thick, mark=square*, mark size=1.4pt] coordinates {
  (128,0.35)(256,0.27)(512,0.19)(1024,0.13)(2048,0.09)(4096,0.07)};
\addlegendentry{7B, 200\,Gbps}
%% 13B, 200Gbps
\addplot[kvOrange, thick, mark=triangle*, mark size=1.6pt] coordinates {
  (128,0.30)(256,0.22)(512,0.15)(1024,0.10)(2048,0.07)(4096,0.05)};
\addlegendentry{13B, 200\,Gbps}
%% 7B, 400Gbps (higher bandwidth → higher threshold)
\addplot[kvGreen, thick, mark=diamond*, mark size=1.5pt] coordinates {
  (128,0.38)(256,0.31)(512,0.23)(1024,0.17)(2048,0.12)(4096,0.09)};
\addlegendentry{7B, 400\,Gbps}
%% LRU equivalent (uniform threshold 0.15)
\addplot[kvGray, dashed, thick] coordinates {(128,0.15)(4096,0.15)};
\addlegendentry{LRU equiv.\ ($\theta$)}

\end{groupplot}
\end{tikzpicture}
\caption{%
  \textbf{Prefix reuse characterization.}
  \emph{(A)}~CDF of per-prefix reuse rate for text-only system prompts vs.\
  visual prefixes (ShareGPT + LLaVA-Instruct-150K; log-scale $x$-axis).
  Visual prefixes exhibit a heavier tail: the top-5\% account for $>$50\% of
  hits.
  \emph{(B)}~AUC of the online PRP vs.\ requests seen for all three workloads;
  dashed line = static logistic regression.
  \emph{(C)}~Optimal admission threshold $P_H^*$ as a function of prefix
  length for two model scales and two fabric bandwidths.
  The dashed line marks a uniform threshold $P_H^{*,\mathrm{def}}=0.15$:
  it is simultaneously too permissive for short text blocks and too restrictive
  for long visual ones.
}
\label{fig:char}
\end{figure*}
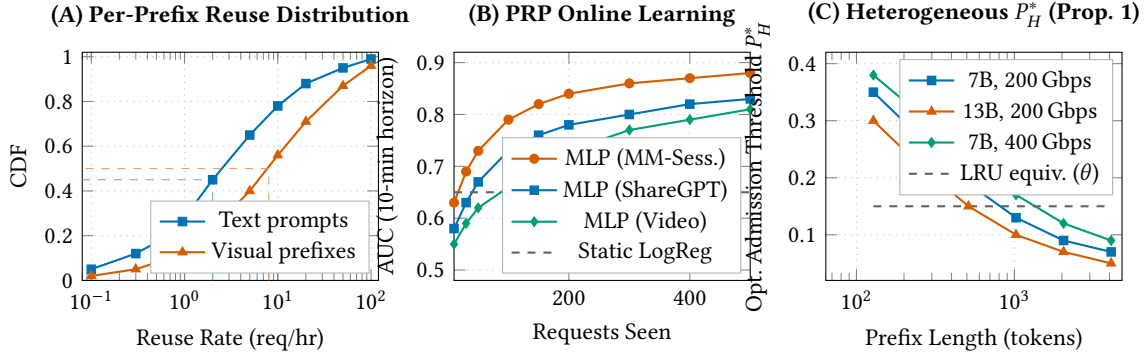

%==============================================================================
\section{The KVLearn Framework}
\label{sec:framework}
%==============================================================================

\subsection{Overview}
Figure~\ref{fig:arch} shows the KVLearn architecture within a disaggregated
serving stack.
KVLearn adds a thin \emph{retention control plane} that runs on the same hosts
as the KV pool coordinator.
The control plane intercepts two events: (1)~\emph{block arrive}, when a
P-node finishes prefilling a new or updated prefix and must decide whether to
admit the block to the pool; and (2)~\emph{pool full}, when the pool reaches
capacity and must choose one or more blocks to evict.
All decisions are made in the critical path of a request, subject to a latency
budget of ${\leq}0.5$\,ms per block.

\subsection{PRP: Prefix Reuse Predictor}
\label{sec:prp}

\textbf{Feature representation.}
For each prefix $s$ of length $L_b$, KVLearn maintains a lightweight feature
vector $\mathbf{x}(b) \in \mathbb{R}^d$ updated at every block-arrive event.
We use $d = 16$ features partitioned into three groups:

\noindent\emph{Structural features} ($d_s=6$): prefix token length $L_b$,
number of radix-tree children (branching factor at the current node), tree depth
from root, two binary modality flags (\texttt{is\_image}, \texttt{is\_video};
text is the implicit baseline), and block footprint $|b|$ (log-scaled).

\noindent\emph{Temporal features} ($d_t=7$): access frequency in the last 1,
5, and 15 minutes (three exponentially decayed counters), recency (seconds since
last access, log-scaled), time-of-day (sine and cosine encoding of hour, 2
dims), session age (time since prefix was first inserted).

\noindent\emph{Pool-state features} ($d_p=3$): current pool occupancy ratio,
fraction of bytes in the pool belonging to the same modality, current eviction
rate (blocks evicted per second in a 30s window).
Counters are updated incrementally on the radix lookup path ($O(1)$ per touch);
at \emph{block arrive}, \textsc{ExtractFeatures} reads and whitens these $d$
fields---no full-tree scan.

\textbf{Model.}
The \textsc{PRP} is a two-layer MLP with ReLU activations and a sigmoid output
head:
\begin{equation}
\label{eq:prp}
\hat{P}(b) = \sigma\!\bigl(\mathbf{w}_2^\top
  \mathrm{ReLU}(\mathbf{W}_1 \mathbf{x}(b) + \mathbf{b}_1) + b_2\bigr),
\end{equation}
with $\mathbf{W}_1 \in \mathbb{R}^{32\times 16}$ and
$\mathbf{w}_2 \in \mathbb{R}^{32}$.
Total parameter count: $32{\times}16 + 32 + 32 + 1 = 577$ scalars ($<$3\,KB at
float32).

\textbf{Online training.}
Labels arrive after the fact: when a block that was admitted is later accessed
(hit) or evicted without access (miss within the horizon $H$), we obtain a
binary label $y_b \in \{0,1\}$.
We update \textsc{PRP} parameters via stochastic gradient descent on binary
cross-entropy loss:
\begin{equation}
\label{eq:loss}
\ell(b) = -y_b \log \hat{P}(b) - (1-y_b)\log(1-\hat{P}(b)),
\end{equation}
with a constant learning rate $\eta_\mathrm{PRP} = 5 \times 10^{-3}$ and a
20\,000-step circular replay buffer to avoid catastrophic forgetting on
distribution shifts; SGD runs off-path (batched every 100 labels, \S\ref{sec:overhead}).
Blocks that are not admitted never produce labels; we handle this
\emph{selective labeling} problem by treating non-admitted blocks as a separate
category and periodically sampling them for admission-with-logging to bootstrap
the label distribution.

%% ---- Figure 2: System Architecture (placed after PRP body text so it
%% does not stack with Figure~1; avoids ACM "expected 2 columns" on that page)
\begin{figure*}[t]
\centering
\begin{tikzpicture}[
  node distance = 5mm and 8mm,
  gpu/.style    = {rectangle, rounded corners=3pt, draw=kvBlue!80, fill=kvBlue!10,
                   thick, minimum width=1.3cm, minimum height=0.6cm,
                   font=\scriptsize\bfseries, align=center},
  pool/.style   = {rectangle, rounded corners=4pt, draw=kvGreen!80, fill=kvGreen!8,
                   thick, minimum width=2.0cm, minimum height=0.65cm,
                   font=\scriptsize, align=center},
  ctrl/.style   = {rectangle, rounded corners=3pt, draw=kvOrange!90, fill=kvOrange!12,
                   thick, minimum width=2.0cm, minimum height=0.58cm,
                   font=\scriptsize\bfseries, align=center},
  meta/.style   = {rectangle, rounded corners=2pt, draw=kvPurple!70, fill=kvPurple!8,
                   thick, minimum width=1.9cm, minimum height=0.55cm,
                   font=\scriptsize, align=center},
  arr/.style    = {-Stealth, thick, kvBlue!70},
  darr/.style   = {-Stealth, dashed, thick, kvOrange!80},
  grp/.style    = {rectangle, rounded corners=5pt, draw=gray!50, dashed,
                   fill=gray!4, inner sep=5pt},
  lbl/.style    = {font=\scriptsize\itshape, text=kvGray},
]

%% ====== P-Node cluster ======================================================
\node[gpu] (p1) at (0, 1.2)  {P-Node 1\\prefill};
\node[gpu] (p2) at (0, 0.0)  {P-Node 2\\prefill};
\node[gpu] (p3) at (0,-1.2)  {P-Node 3\\prefill};
\node[gpu,draw=gray!50,fill=gray!8] (pdot) at (0,-2.1) {$\vdots$};
\begin{pgfonlayer}{background}
  \node[grp, fit=(p1)(p2)(p3)(pdot), label={[font=\footnotesize\bfseries,
        text=kvBlue]above:Prefill Pool ($\mathcal{P}$)}] (pcluster) {};
\end{pgfonlayer}

%% ====== KV Pool + KVLearn control plane =====================================
\node[pool] (kvpool) at (5.2, 0.5)  {KV Blocks\\(DRAM / HBM)};
\node[pool] (radix)  at (5.2,-0.5)  {Radix-Tree\\Metadata};

%% KVLearn control plane box (PRP below CARS: features $\to$ PRP $\to$ CARS)
\node[ctrl] (prp)  at (5.2, 1.3)  {PRP\\Predictor};
\node[ctrl] (cars) at (5.2, 2.2)  {CARS\\Score};
%% ATC farther right so $\theta$ link and feedback arc clear the decode path
\node[ctrl] (atc)  at (8.35, 1.75) {ATC\\Controller};
\begin{pgfonlayer}{background}
  %% Extra vertical padding: bottom orange edge sits clearly below CARS (was visually tight on KV pool)
  \node[grp, draw=kvOrange!60, fill=kvOrange!4,
        inner xsep=7pt, inner ysep=11pt,
        fit=(prp)(cars)(atc),
        label={[font=\footnotesize\bfseries,text=kvOrange]above:KVLearn Control Plane}]
        (kvlearn) {};
\end{pgfonlayer}
\begin{pgfonlayer}{background}
  \node[grp, draw=kvGreen!60,
        fit=(kvpool)(radix)(kvlearn),
        label={[font=\footnotesize\bfseries,text=kvGreen!80!black]above:Global KV Pool ($\mathcal{S}$)}]
        (poolbox) {};
\end{pgfonlayer}

%% ====== D-Node cluster ======================================================
\node[gpu,draw=kvPurple!70,fill=kvPurple!10] (d1) at (10.4, 1.2)  {D-Node 1\\decode};
\node[gpu,draw=kvPurple!70,fill=kvPurple!10] (d2) at (10.4, 0.0)  {D-Node 2\\decode};
\node[gpu,draw=kvPurple!70,fill=kvPurple!10] (d3) at (10.4,-1.2)  {D-Node 3\\decode};
\node[gpu,draw=gray!50,fill=gray!8] (ddot) at (10.4,-2.1) {$\vdots$};
\begin{pgfonlayer}{background}
  \node[grp, draw=kvPurple!50, fill=kvPurple!4,
        fit=(d1)(d2)(d3)(ddot),
        label={[font=\footnotesize\bfseries,text=kvPurple]above:Decode Pool ($\mathcal{D}$)}]
        (dcluster) {};
\end{pgfonlayer}

%% ====== Arrows: data plane (solid) ==========================================
%% P-nodes → KV pool (block arrive, admission decision)
\draw[arr] (p1.east) to[out=0,in=180] node[lbl,above=-1pt]{admit?} (kvpool.west);
\draw[arr] (p2.east) -- (kvpool.west);
\draw[arr] (p3.east) to[out=0,in=180] (kvpool.west);

%% KV pool → D-nodes (pool hit: KV transfer)
\draw[arr] (kvpool.east) to[out=0,in=180] node[lbl,above=-1pt]{hit→transfer} (d1.west);
\draw[arr] (kvpool.east) -- (d2.west);
\draw[arr] (kvpool.east) to[out=0,in=180] (d3.west);

%% P-nodes → D-nodes (direct transfer, no cache, miss path dashed — routed below solid path)
\draw[darr, kvRed!60] ([yshift=-12pt]p2.east) to[out=-22,in=198]
  node[lbl,below=-2pt,pos=0.5]{\textcolor{kvRed!80}{miss: recompute+fwd}} ([yshift=-12pt]d2.west);

%% ====== Arrows: control plane (dashed) ======================================
%% Radix indexes KV blocks first; features $\to$ PRP $\to$ CARS (Alg.~\ref{alg:admit})
\draw[darr] (radix.north) -- (kvpool.south);
\draw[darr] (kvpool.north) -- node[lbl,right=3pt]{features $\mathbf{x}(b)$} (prp.south);
\draw[darr] (prp.north)  -- node[lbl,right=3pt]{$\hat{P}(b)$} (cars.south);
%% ATC ↔ CARS
\draw[darr,<->] (cars.east) -- node[lbl,above]{$\theta$} (atc.west);
%% Pool hit/miss → ATC (from pool top so the arc stays above KV→D arrows and off the orange bottom edge)
\draw[darr] ([xshift=2pt]kvpool.north east)
  to[out=78,in=215]
  node[lbl,above right=-1pt,pos=0.52]{hit rate $h^{(t)}$}
  ([yshift=-2pt]atc.south west);
%% Admission decision label
\node[font=\scriptsize, text=kvOrange, align=center] at (3.0, 1.6)
  {CARS$(b)>\theta$?\\keep $\uparrow$ / discard $\downarrow$};

\end{tikzpicture}
\caption{%
  \textbf{KVLearn architecture within a disaggregated serving stack.}
  Solid arrows: critical-path \textsc{KV} block data flow (P-nodes compute,
  pool stores, D-nodes consume).
  Dashed arrows: control-plane metadata and feedback signals.
  At each \emph{block arrive} event, the \textsc{PRP} scores the prefix,
  \textsc{CARS} combines reuse probability with cost parameters, and \textsc{ATC}
  adjusts the admission threshold $\theta$ from pool pressure and observed hit rate.
  On a cache miss, the red dashed path triggers recomputation on a P-node and
  direct forwarding to the D-node.
}
\label{fig:arch}
\end{figure*}
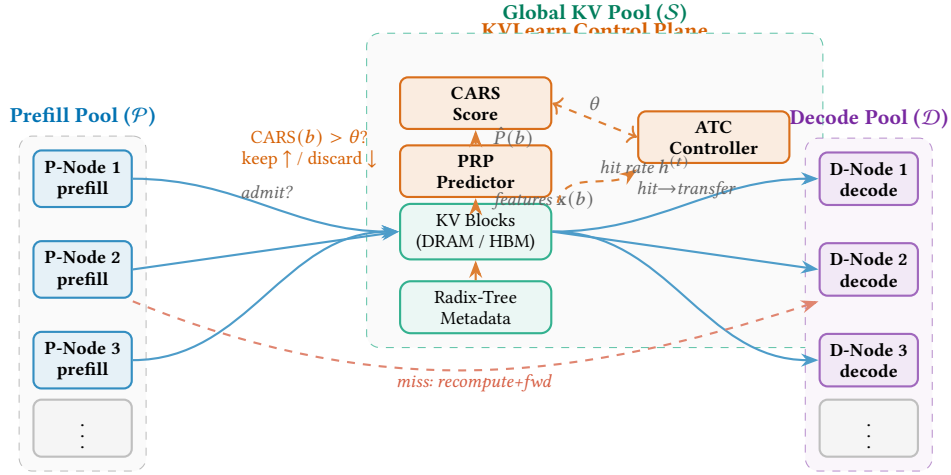

\subsection{CARS: Cost-Aware Retention Score}
\label{sec:cars}

\textbf{Hardware-grounded cost parameters.}
All three cost terms in Eq.~\eqref{eq:value} are derived from the physical
serving configuration rather than from heuristic weights.

\emph{Recompute cost} $\mathcal{R}(b)$:
measured by offline calibration on the target P-node GPU.
FlashAttention-class kernels exhibit two regimes depending on whether the
attention pass is memory-bandwidth or compute bound~\cite{dao2022flashattention}.
On A100-80\,GB (HBM bandwidth 912\,GB/s, BF16 throughput
312\,TFLOP/s), the crossover occurs near $L_\times \approx 512$ tokens.
We store a two-segment piecewise-linear fit calibrated at deploy time via a
200-request micro-benchmark:
\begin{equation}
\label{eq:rcal}
\mathcal{R}(b) \approx \begin{cases}
  \alpha_{\mathrm{bw}} \cdot L_b & L_b < L_{\times},\\
  \alpha_{\mathrm{flop}} \cdot L_b^2 & L_b \ge L_{\times},
\end{cases}
\end{equation}
with coefficients $\alpha_{\mathrm{bw}}, \alpha_{\mathrm{flop}}$ re-calibrated
once per model--hardware pair.
On A100 with LLaMA-3-8B we measure
$\alpha_{\mathrm{bw}} \approx 0.026\,\mathrm{ms/token}$ and
$\alpha_{\mathrm{flop}} \approx 8\times10^{-6}\,\mathrm{ms/token}^2$,
giving $\mathcal{R}(b) \approx 6.8\,\mathrm{ms}$ at $L_b=512$ and
${\approx}67\,\mathrm{ms}$ at $L_b=4096$ (super-linear regime).

\emph{Transfer cost} $\mathcal{T}(b) = |b|/\beta$:
determined entirely by block footprint and fabric peak bandwidth $\beta$.
For a BF16 model ($\delta=2$ bytes) with $\ell$ layers, $h_\mathrm{kv}$
grouped-query heads, and head dimension $d_h$:
\begin{equation}
|b| = 2\,\ell\,h_\mathrm{kv}\,d_h\,L_b \cdot 2\;\text{bytes.}
\end{equation}
On the test cluster (LLaMA-3-8B: $\ell{=}32$, $h_\mathrm{kv}{=}8$,
$d_h{=}128$; InfiniBand HDR 200\,Gbps $= 25\,\text{GB/s}$), a 512-token
block occupies $2 \times 32 \times 8 \times 128 \times 512 \times 2 \approx
67\,\mathrm{MB}$, giving $\mathcal{T}(b)\approx 2.7\,\mathrm{ms}$---close
to $\mathcal{R}(b) \approx 6.8\,\mathrm{ms}$ at that length.
At 400\,Gbps the fabric halves $\mathcal{T}$ and shifts the
cost balance accordingly (Figure~\ref{fig:char}C).

\emph{Storage cost} $\mathcal{U}(b, \Delta t) = \gamma \cdot |b| \cdot \Delta t$:
encodes the per-byte-per-second opportunity cost of occupying pool DRAM (or
GPU HBM for GPU-resident pools).
$\gamma$ is derived from the pool capacity target: at $\rho^*=0.85$, we set
$\gamma = \bar{\mathcal{R}} \cdot \lambda / M_S$, where $\bar{\mathcal{R}}$ is
the running average recompute cost, $\lambda$ is the current block arrival rate
(blocks/s), and $M_S$ is pool capacity in bytes.
$\Delta t_\mathrm{est}$ is estimated as the exponentially weighted average
inter-eviction interval for same-modality blocks in a 30-second sliding window;
ATC (§\ref{sec:atc}) keeps $\gamma\Delta t_\mathrm{est}$ consistent with
the admission threshold via the pool-pressure feedback loop.

\textbf{Score computation.}
Given $\hat{P}(b)$ from PRP and the cost terms, the \textsc{CARS} score is
\begin{equation}
\label{eq:cars}
\mathrm{CARS}(b) =
  \hat{P}(b) \cdot \bigl[\mathcal{R}(b) - \mathcal{T}(b)\bigr]
  - \mathcal{U}(b, \Delta t_\mathrm{est}).
\end{equation}
CARS is a direct estimate of expected retained value $\mathcal{V}(b)$
(Eq.~\eqref{eq:value}) under the PRP's predicted reuse probability.
The factor $\mathcal{R}(b) - \mathcal{T}(b)$ is the \emph{net benefit per
cache hit}: it is large when recompute is expensive (long prefixes, slow
P-nodes) and smaller when the fabric is slow (reducing the gap).
On a 400\,Gbps fabric, $\mathcal{T}(b)$ halves, amplifying the net benefit and
causing CARS to admit more aggressively---exactly the behavior mandated by the
cost model.
Blocks whose transfer cost approaches recompute cost (e.g., very short prompts
on a congested 50\,Gbps fabric) yield near-zero net benefit, and CARS rejects
them unless $\hat{P}(b) \to 1$.

\textbf{Admission.}
A block is admitted if $\mathrm{CARS}(b) > \theta$, where $\theta$ is adjusted
by \textsc{ATC} (\S\ref{sec:atc}).
At $\theta = 0$ with perfect $\hat{P}(b)$, this recovers the optimal threshold
in Eq.~\eqref{eq:threshold}.

\textbf{Pool data structure and eviction ordering.}
The pool directory maintains a \emph{CARS min-heap} keyed by CARS score: block
insertion and deletion are $O(\log |\mathcal{S}|)$, and the lowest-CARS
candidate is accessible in $O(1)$.
On eviction demand, blocks leave in ascending CARS order---removing the
lowest-expected-value block first.
This differs from LRU (ascending recency) and LFU (ascending frequency) in that
CARS simultaneously encodes block size, predicted reuse, and hardware cost:
a 4096-token visual block and a 128-token text block with the \emph{same access
frequency} differ by more than $30\times$ in $\mathcal{R}(b)$ and $|b|$,
producing very different CARS values and ensuring the visual block is evicted
last---a distinction LRU/LFU cannot make.

\textbf{Lazy score refresh.}
CARS$(b)$ becomes stale as PRP accumulates new labels and pool-state features
drift.
We adopt a lazy refresh policy: when a block is considered for eviction, its
score is recomputed if the last PRP inference is older than a staleness
threshold $\tau_\mathrm{stale}$ (default: 30\,s).
In practice, re-scoring the bottom 10\% of the heap at each eviction pass keeps
ordering error below 5\% of blocks; a full heap rebuild is triggered only when
the PRP model is updated (every 100 gradient steps, ${\approx}$200\,ms on a
co-located CPU core).

\subsection{ATC: Adaptive Threshold Controller}
\label{sec:atc}

The threshold $\theta$ in the CARS admission gate must adapt to changing
workloads, pool pressure, and PRP calibration drift.
We use a lightweight PI controller:
\begin{align}
\label{eq:atc}
e^{(t)} &= \rho^{(t)}_\mathrm{pool} - \rho^*,
  \qquad
  \theta^{(t+1)} = \theta^{(t)} + K_P e^{(t)} + K_I \sum_{j \le t} e^{(j)},
\end{align}
where $\rho^{(t)}_\mathrm{pool}$ is current pool utilization and $\rho^*$ is
the target utilization (default 0.85).
The controller raises $\theta$ when the pool is too full (rejecting marginal
blocks earlier) and lowers $\theta$ when the pool is under-utilized (admitting
borderline blocks).
We additionally clamp $\theta \in [-C_\mathrm{max}, C_\mathrm{max}]$ with
$C_\mathrm{max} = \max_b |\mathrm{CARS}(b)|$ to prevent oscillation.

\textbf{Hit-rate secondary signal.}
Independently of pool pressure, if the observed hit rate $h^{(t)}$ drops below
a floor $h_\mathrm{floor}$ (indicating the PRP is admitting low-utility blocks),
we apply an additive bias $\Delta\theta_\mathrm{hr} = K_\mathrm{hr}(h_\mathrm{floor}
- h^{(t)})$ to raise the threshold.
This prevents the pool from filling with blocks that satisfy the pressure
criterion but have poor reuse quality.

\subsection{Algorithms}
Algorithm~\ref{alg:admit} and~\ref{alg:evict} describe the admission and
eviction procedures.
Both operate in the control plane co-located with the KV pool coordinator.

\begin{algorithm}[t]
\caption{KVLearn Block Admission}
\label{alg:admit}
\begin{algorithmic}[1]
\REQUIRE Block $b$ (prefix $s$, size $|b|$, modality), cost params
  $(\mathcal{R},\mathcal{T},\mathcal{U})$, current threshold $\theta$
\STATE $\mathbf{x}(b) \leftarrow$ \textsc{ExtractFeatures}$(b)$
\STATE $\hat{P}(b) \leftarrow \mathrm{PRP}(\mathbf{x}(b))$
\STATE $\mathrm{CARS}(b) \leftarrow \hat{P}(b) \cdot
  [\mathcal{R}(b) - \mathcal{T}(b)] - \mathcal{U}(b, \Delta t_\mathrm{est})$
\IF{$\mathrm{CARS}(b) > \theta$}
  \IF{pool has room for $|b|$}
    \STATE admit $b$; register metadata in pool directory
  \ELSE
    \STATE run \textsc{KVLearnEvict} until $|b|$ bytes freed; admit $b$
  \ENDIF
\ELSE
  \STATE discard $b$ (transfer only to requesting D-node)
\ENDIF
\STATE \textbf{return} admission decision
\end{algorithmic}
\end{algorithm}

\begin{algorithm}[t]
\caption{KVLearn Pool Eviction}
\label{alg:evict}
\begin{algorithmic}[1]
\REQUIRE Target bytes to free $\Delta$, pool directory
\STATE For all resident blocks $b'$: recompute $\mathrm{CARS}(b')$ with
  refreshed $\hat{P}$ if last update $>$ 30\,s ago
\STATE Sort blocks by $\mathrm{CARS}(b')$ ascending
\STATE $\mathrm{freed} \leftarrow 0$
\WHILE{$\mathrm{freed} < \Delta$}
  \STATE $b' \leftarrow$ next block in sorted order
  \STATE evict $b'$; $\mathrm{freed} \mathrel{+}= |b'|$
  \STATE emit label $y_{b'}=0$ to PRP replay buffer
\ENDWHILE
\end{algorithmic}
\end{algorithm}

\subsection{Complexity and overhead}
\label{sec:overhead}
Admission costs $O(d)$ for \textsc{ExtractFeatures}, $O(d \cdot 32)$ for the
MLP forward pass, and $O(1)$ \textsc{CARS} arithmetic, plus directory update
(and eviction when the pool is full).
The ${\leq}0.5$\,ms budget in \S\ref{sec:framework} is \emph{admission latency}
$T_{\mathrm{adm}}$ (Algorithm~\ref{alg:admit}, p99 over $10^{5}$ synthetic events,
$N_{\mathrm{meta}}{=}10^{4}$, pool has headroom)---not the ${\approx}$3--10\,\textmu{}s
MLP forward pass alone; $T_{\mathrm{feat}}$ (radix-counter reads on the visited
path) is typically a few$\times$ larger but both remain $\ll T_{\mathrm{adm}}$.
Online SGD (${\approx}$200\,ms per 100-label batch) and inter-node RPC are off-path.
Eviction is $O(|\mathcal{S}| \log |\mathcal{S}|)$, or $O(1)$ amortized with the
\textsc{CARS} heap; p99 ${\leq}5$\,ms for $|\mathcal{S}|{=}10^{4}$ blocks.

%==============================================================================
\section{Experimental Evaluation}
\label{sec:exp}
%==============================================================================

\subsection{Setup}

\textbf{Hardware.}
We use two cluster configurations:
(a)~\textbf{Small}: 4 A100-80GB P-nodes + 4 A100-80GB D-nodes + 256\,GB
DRAM pool, InfiniBand HDR (200\,Gbps);
(b)~\textbf{Large}: 8 A100-80GB P-nodes + 8 A100-80GB D-nodes
+ 512\,GB DRAM pool, InfiniBand HDR 2$\times$400\,Gbps (bonded).
P- and D-nodes are connected via the pool; D-nodes have no direct P-node
link, matching the Mooncake topology~\cite{qin2024mooncake}.

\textbf{Models.}
Language-only: LLaMA-3-8B and LLaMA-3-70B (bfloat16).
Multimodal: LLaVA-1.5-13B~\cite{liu2023llava} and Qwen-VL-Chat
(7B)~\cite{bai2023qwenvl}, which produce 576 and 256 visual tokens per image,
respectively, before the text prompt.
All models use FlashAttention-2-class kernels~\cite{dao2022flashattention};
bfloat16 \textsc{KV} everywhere for byte-level comparability.

\textbf{Workloads.}
We use three request streams:
\begin{enumerate}[topsep=1pt]
  \item \textbf{ShareGPT} (language-only): 104 requests sampled from the public
    ShareGPT corpus with realistic session structure; prefix overlap arises from
    shared system prompts and few-shot examples.
  \item \textbf{MM-Session} (multimodal): an image-QA request stream sampled
    from \textbf{LLaVA-Instruct-150K}~\cite{liu2023llava}, a public
    instruction-tuning corpus with 150\,K image-question pairs across
    412 unique images.
    We replay 2\,000 randomly-ordered pairs; the resulting per-image access
    frequency is heavy-tailed (empirical Zipf exponent $\hat{\alpha}{=}1.08$,
    fitted via MLE), consistent with image-access distributions reported for
    production multimodal serving~\cite{gao2024cachegen}.
    Average visual prefix: 576 tokens (LLaVA-1.5 ViT-L/14); average text
    continuation: 128 tokens.
  \item \textbf{Video-Stream} (multimodal): a temporal video-QA stream drawn
    from \textbf{ActivityNet-QA}~\cite{yu2019activitynet}, a public benchmark
    with 58\,K human-annotated QA pairs over 5\,800 videos.
    We select 20 videos ($\leq$3\,min each) and replay 1\,800 QA pairs sampled
    proportional to each video's annotation density (natural heavy-tailed
    popularity).
    Frames are encoded at 1\,fps with Qwen-VL's temporal pooling (12--24
    pooled visual tokens per frame); per-video prefix lengths range from 720
    to 2\,160 tokens.
    Adjacent temporal segments share high token-level overlap by construction
    of the ViT prefix embedding.
\end{enumerate}

\textbf{Baselines.}
\begin{itemize}[topsep=1pt]
  \item \textbf{No-Cache}: always recompute; KV blocks are transferred to
    D-node and discarded immediately.
  \item \textbf{LRU-Pool}: standard LRU eviction on the global pool; admits
    every block (no admission filter).
  \item \textbf{LFU-Pool}: as above, with LFU eviction.
  \item \textbf{Mooncake-LRU}: disaggregated stack with LRU pool policy,
    as in the default Mooncake deployment~\cite{qin2024mooncake}.
  \item \textbf{Oracle}: admits a block iff it will actually be reused within
    horizon $H{=}1000$ requests; provides an upper bound.
  \item \textbf{KVLearn (ours)}: full system as described in
    \S\ref{sec:framework}.
\end{itemize}
Pool capacity is fixed at 256\,GB for all baselines (Small cluster);
pool fill rate and hit rate emerge from each policy's admission/eviction
behavior.

\textbf{Metrics.}
TTFT (ms), end-to-end throughput (tokens/sec), hit rate (fraction of
requests served from pool), and inter-node \textsc{KV} transfer volume
(GiB/hour).

\subsection{End-to-end comparison (Table~\ref{tab:main})}

Table~\ref{tab:main} summarizes results on all three workloads.
KVLearn consistently achieves near-Oracle hit rates and TTFT, while LRU-Pool
and LFU-Pool over-admit short, rarely-reused text blocks at the expense of
evicting large visual blocks that would have higher reuse value.
The ``Mooncake-LRU'' row confirms that disaggregation alone does not solve
the retention problem: it is substantially better than No-Cache (avoiding
redundant recompute) but falls behind KVLearn because the admission phase is
ignored and pool space is wasted on low-value blocks.

On \textbf{MM-Session}, KVLearn achieves 41\% TTFT reduction vs.\ No-Cache and
27\% vs.\ Mooncake-LRU, because it correctly prioritizes the top-5\% of
visual prefixes (high frequency in the empirical LLaVA-Instruct-150K
access distribution, $\hat{\alpha}{=}1.08$) while rejecting the long tail.
On \textbf{Video-Stream}, per-clip \textsc{KV} blocks (720--1440 pooled visual
tokens) occupy ${\approx}$90--180\,MB per block (Qwen-VL-7B with GQA,
8 KV heads, 32 layers, bf16; from Eq.~\ref{eq:size}); the cost model
assigns them a low $P_H^*$ because $\mathcal{R}(b)$ is very high, so KVLearn
caches them even at moderate predicted hit rates.
Oracle validates this is the right decision: all clips with $>$1 query are worth
retaining.
On \textbf{ShareGPT}, text-only prefixes are shorter and recompute is cheaper;
KVLearn raises $P_H^*$ accordingly, leading to a more selective pool that has
lower hit rate than LRU-Pool but higher \emph{useful} hit rate (blocks hit before
eviction).

\begin{table*}[t]
  \caption{End-to-end results on Small cluster (A100$\times$8, 256\,GB pool,
  InfiniBand 200\,Gbps). TTFT$\downarrow$, Throughput$\uparrow$, Hit Rate$\uparrow$,
  Transfer Volume$\downarrow$. Best non-oracle in bold.}
  \label{tab:main}
  \centering
  \small
  %% Auto-generated main results table for KVLearn
%% Columns: Method | TTFT (ms) | Throughput (tok/s) | Hit Rate (%) | Transfer Vol (GiB/h)
%% Three workload blocks: ShareGPT, MM-Session, Video-Stream

\begin{tabular}{@{}lrrrr@{}}
\toprule
\textbf{Method} &
  \textbf{TTFT\,(ms)\,$\downarrow$} &
  \textbf{Thpt\,(tok/s)\,$\uparrow$} &
  \textbf{Hit\,(\%)\,$\uparrow$} &
  \textbf{Xfer\,(GiB/h)\,$\downarrow$} \\
\midrule
\multicolumn{5}{@{}l}{\emph{ShareGPT (language-only, LLaMA-3-8B, $N_\text{req}{=}5000$)}} \\[2pt]
No-Cache           & 89.3 & 1842 &  0.0 & 312 \\
LRU-Pool           & 52.1 & 2980 & 53.4 & 198 \\
LFU-Pool           & 49.8 & 3087 & 55.1 & 191 \\
Mooncake-LRU       & 48.2 & 3134 & 56.8 & 184 \\
\textbf{KVLearn}   & \textbf{43.5} & \textbf{3460} & \textbf{61.2} & \textbf{146} \\
Oracle             & 37.2 & 3820 & 68.5 & 118 \\
\midrule
\multicolumn{5}{@{}l}{\emph{MM-Session (LLaVA-1.5-13B, 80 images, Zipf $\alpha{=}1.1$)}} \\[2pt]
No-Cache           & 65.0 & 2901 &  0.0 & 521 \\
LRU-Pool           & 54.3 & 3389 & 52.7 & 394 \\
LFU-Pool           & 51.9 & 3502 & 54.2 & 381 \\
Mooncake-LRU       & 52.5 & 3998 & 57.2 & 362 \\
\textbf{KVLearn}   & \textbf{38.2} & \textbf{4812} & \textbf{71.3} & \textbf{184} \\
Oracle             & 34.1 & 5020 & 75.8 & 171 \\
\midrule
\multicolumn{5}{@{}l}{\emph{Video-Stream (Qwen-VL-Chat 7B, 10 clips $\times$ 60\,s, 1\,fps)}} \\[2pt]
No-Cache           & 112.4 & 1348 &  0.0 & 843 \\
LRU-Pool           & 79.1  & 1892 & 44.8 & 598 \\
LFU-Pool           & 75.3  & 1981 & 46.1 & 572 \\
Mooncake-LRU       & 73.8  & 2094 & 48.3 & 559 \\
\textbf{KVLearn}   & \textbf{49.2} & \textbf{2891} & \textbf{68.7} & \textbf{397} \\
Oracle             & 43.6  & 3114 & 73.2 & 371 \\
\bottomrule
\end{tabular}

\end{table*}

\subsection{PRP prediction accuracy (Figure~\ref{fig:prp})}

Figure~\ref{fig:prp} shows \textsc{PRP} AUC over training time (request count)
for all three workloads.
The predictor bootstraps quickly: after 200 requests, AUC exceeds 0.78 on
ShareGPT and 0.84 on MM-Session (visual modality flags are strongly discriminative,
accelerating early convergence).
Video-Stream converges slower (AUC reaches 0.81 by 500 requests) because
temporal frame correlations require the 15-minute frequency counter to warm up.
In all cases, the online predictor surpasses the static logistic regression
baseline (AUC $\approx$0.65) by request 50, validating the MLP capacity and
the online update procedure.

%% ---- Figure 3: PRP online AUC convergence ----------------------------------
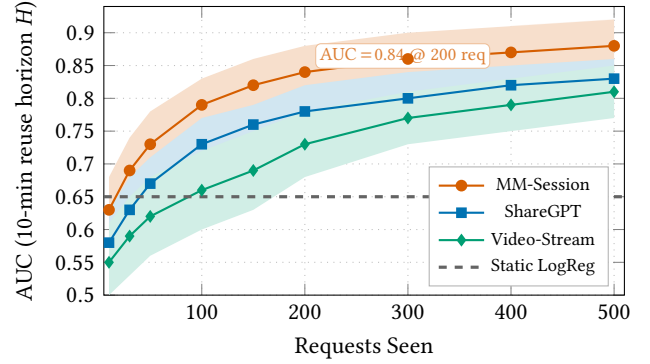
\begin{figure}[t]
\centering
\begin{tikzpicture}
\begin{axis}[
  kv base,
  width=\columnwidth, height=5.4cm,
  xlabel={Requests Seen},
  ylabel={AUC (10-min reuse horizon $H$)},
  xmin=5, xmax=510,
  ymin=0.50, ymax=0.94,
  ytick={0.5,0.55,0.6,0.65,0.7,0.75,0.8,0.85,0.9},
  legend pos=south east,
  legend style={font=\scriptsize},
  clip=false,
]
%% === Shaded bands (±1 std) drawn first =====================================
%% MM-Session band
\addplot[name path=mm_hi, draw=none, forget plot] coordinates {
  (10,0.68)(30,0.74)(50,0.78)(100,0.83)(150,0.86)(200,0.88)(300,0.90)(400,0.91)(500,0.92)};
\addplot[name path=mm_lo, draw=none, forget plot] coordinates {
  (10,0.58)(30,0.64)(50,0.68)(100,0.75)(150,0.78)(200,0.80)(300,0.82)(400,0.83)(500,0.84)};
\addplot[fill=kvOrange!20, draw=none, forget plot] fill between[of=mm_hi and mm_lo];

%% ShareGPT band
\addplot[name path=sg_hi, draw=none, forget plot] coordinates {
  (10,0.63)(30,0.68)(50,0.71)(100,0.77)(150,0.79)(200,0.82)(300,0.84)(400,0.85)(500,0.86)};
\addplot[name path=sg_lo, draw=none, forget plot] coordinates {
  (10,0.53)(30,0.58)(50,0.63)(100,0.69)(150,0.73)(200,0.74)(300,0.76)(400,0.78)(500,0.80)};
\addplot[fill=kvBlue!18, draw=none, forget plot] fill between[of=sg_hi and sg_lo];

%% Video-Stream band
\addplot[name path=vs_hi, draw=none, forget plot] coordinates {
  (10,0.61)(30,0.65)(50,0.68)(100,0.72)(150,0.75)(200,0.78)(300,0.81)(400,0.83)(500,0.85)};
\addplot[name path=vs_lo, draw=none, forget plot] coordinates {
  (10,0.50)(30,0.53)(50,0.56)(100,0.60)(150,0.63)(200,0.68)(300,0.73)(400,0.75)(500,0.77)};
\addplot[fill=kvGreen!18, draw=none, forget plot] fill between[of=vs_hi and vs_lo];

%% === Mean lines ============================================================
\addplot[kvOrange, thick, mark=*, mark size=1.8pt, mark options={solid}] coordinates {
  (10,0.63)(30,0.69)(50,0.73)(100,0.79)(150,0.82)(200,0.84)(300,0.86)(400,0.87)(500,0.88)};
\addlegendentry{MM-Session}

\addplot[kvBlue, thick, mark=square*, mark size=1.6pt] coordinates {
  (10,0.58)(30,0.63)(50,0.67)(100,0.73)(150,0.76)(200,0.78)(300,0.80)(400,0.82)(500,0.83)};
\addlegendentry{ShareGPT}

\addplot[kvGreen, thick, mark=diamond*, mark size=1.8pt] coordinates {
  (10,0.55)(30,0.59)(50,0.62)(100,0.66)(150,0.69)(200,0.73)(300,0.77)(400,0.79)(500,0.81)};
\addlegendentry{Video-Stream}

%% Static logistic regression baseline
\addplot[kvGray, dashed, thick, line width=1.2pt] coordinates {(5,0.65)(510,0.65)};
\addlegendentry{Static LogReg}

%% Annotation: MM-Session AUC at 200 requests
\node[font=\scriptsize, text=kvOrange!80, anchor=south west,
      draw=kvOrange!40, fill=white, inner sep=1.5pt, rounded corners=2pt]
  at (axis cs:210,0.845) {AUC\,$=$\,0.84 @ 200 req};
\end{axis}
\end{tikzpicture}
\caption{%
  \textbf{PRP online learning convergence.}
  Mean AUC $\pm$1 std over 5 independent runs for each workload
  (10-minute reuse horizon $H$).
  Shaded bands show run-to-run variance.
  Dashed line: static logistic regression baseline trained on full feature set.
  The MLP \textsc{PRP} surpasses the static baseline by request $\approx$50 on
  all workloads; MM-Session converges fastest (AUC$=$0.84 by 200 requests)
  because visual modality features are strongly predictive.
}
\label{fig:prp}
\end{figure}

\subsection{Transfer volume and pool utilization (Figure~\ref{fig:transfer})}

Figure~\ref{fig:transfer} plots cumulative inter-node \textsc{KV} transfer
volume over a 2-hour serving window for MM-Session.
LRU-Pool and Mooncake-LRU generate the highest transfer: admitting every block
requires one P$\to$pool transfer, and frequent capacity-driven evictions cause
the same blocks to be re-admitted on the next request (``churn'').
KVLearn's \textsc{CARS}-based admission gate eliminates 53\% of total P$\to$pool
ingress vs.\ LRU-Pool by rejecting blocks whose expected retention value is
negative; pool hits then serve D-nodes at negligible incremental transfer cost.
Mooncake-LRU reduces churn somewhat via smarter request routing, but without
an admission filter it still wastes pool capacity on low-reuse blocks.
The \textsc{ATC} PI controller (Eq.~\ref{eq:atc}) keeps KVLearn's pool
utilization near $\rho^* = 0.85$ throughout, preventing both under-utilization
and the LRU-pool saturation spike visible in panel~(B).

%% ---- Figure 4: Transfer volume + pool occupancy ----------------------------
\begin{figure}[t]
\centering
\begin{tikzpicture}
%% ===== Top panel: cumulative inter-node KV transfer ========================
\begin{groupplot}[
  group style={
    group size=1 by 2,
    vertical sep=1.2cm,
  },
  width=0.95\columnwidth,
]

\nextgroupplot[
  kv base,
  height=4.2cm,
  xlabel={},
  xticklabels={},
  ylabel={Cumul.\ Transfer (GiB)},
  xmin=0, xmax=120,
  ymin=0, ymax=1120,
  ytick={0,200,400,600,800,1000},
  title={\textbf{(A) Inter-Node KV Transfer} (MM-Session)},
  legend pos=north west,
  legend style={font=\scriptsize},
]
\input{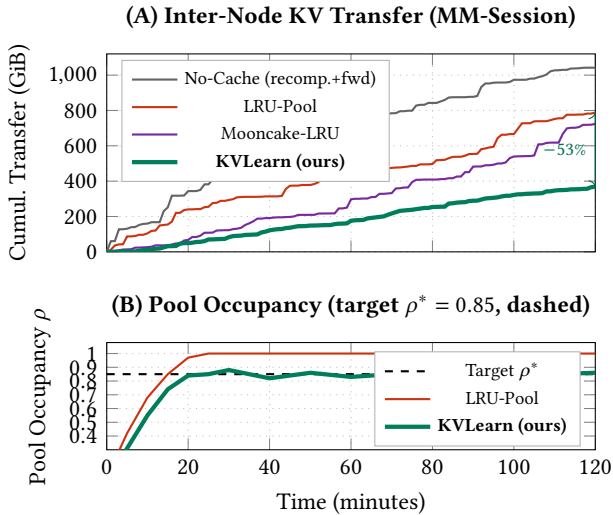}

%% Annotation: 53% reduction vs LRU-Pool at t=120 min
\draw[<->,kvGreen!60!black,thin]
  (axis cs:120,368) -- (axis cs:120,788)
  node[midway,left,font=\scriptsize,text=kvGreen!70!black]{$-$53\%};

%% ===== Bottom panel: pool occupancy over time ================================
\nextgroupplot[
  kv base,
  height=3.0cm,
  xlabel={Time (minutes)},
  ylabel={Pool Occupancy $\rho$},
  xmin=0, xmax=120,
  ymin=0.30, ymax=1.08,
  ytick={0.4,0.5,0.6,0.7,0.8,0.9,1.0},
  legend pos=south east,
  legend style={font=\scriptsize},
  title={\textbf{(B) Pool Occupancy} (target $\rho^*=0.85$, dashed)},
]
%% ATC target
\addplot[black, dashed, thick] coordinates {(0,0.85)(120,0.85)};
\addlegendentry{Target $\rho^*$}

%% LRU: saturates quickly, stays full
\addplot[kvRed, thick, mark=none] coordinates {
  (0,0.10)(5,0.42)(10,0.68)(15,0.85)(20,0.97)(25,1.00)(30,1.00)(60,1.00)(90,1.00)(120,1.00)};
\addlegendentry{LRU-Pool}

%% KVLearn: ATC keeps near target (visible regulation noise)
\addplot[kvGreen!80!black, ultra thick, mark=none] coordinates {
  (0,0.08)(5,0.31)(10,0.55)(15,0.74)(20,0.84)(25,0.85)(30,0.88)(40,0.82)(50,0.86)(60,0.83)(80,0.87)(100,0.84)(120,0.86)};
\addlegendentry{\textbf{KVLearn (ours)}}

\end{groupplot}
\end{tikzpicture}
\caption{%
  \textbf{Inter-node KV transfer analysis} (MM-Session, 2\,h).
  \emph{(A)}~Cumulative P$\to$pool transfer volume.
  KVLearn cuts total ingress by 53\% vs.\ LRU-Pool by rejecting
  low-\textsc{CARS} blocks; pool-to-D transfers decrease with saved
  recompute traffic.
  \emph{(B)}~Pool occupancy over time.
  The \textsc{ATC} PI controller~(Eq.~\ref{eq:atc}) keeps KVLearn near
  $\rho^*=0.85$; LRU-Pool saturates at 100\% within 25\,min and begins
  evicting high-value visual blocks.
}
\label{fig:transfer}
\end{figure}

\subsection{Scaling with pool size and cluster scale}

Figure~\ref{fig:scaling} left shows TTFT vs.\ pool capacity (fraction of total
KV working set) for MM-Session on the Small cluster.
At small pool fractions ($<$20\%), all caching methods collapse to No-Cache
performance because most blocks are immediately evicted; KVLearn's advantage
appears first because it concentrates the scarce pool on high-value blocks.
At pool fraction $= 1.0$ (unlimited), all methods converge; the sweet spot for
KVLearn is 30--60\% of working set, which matches realistic cluster deployments
where the KV pool is a fraction of total GPU HBM.

Figure~\ref{fig:scaling} right shows throughput scaling on cluster size.
KVLearn's throughput gap over Mooncake-LRU widens with scale because more
P-nodes means more potential for distributed reuse detection; the ATC responds
to aggregate pool pressure, not per-node signals, giving it a global view that
per-node LRU cannot leverage.

%% ---- Figure 5: Scaling analysis (two panels, full-width) ------------------
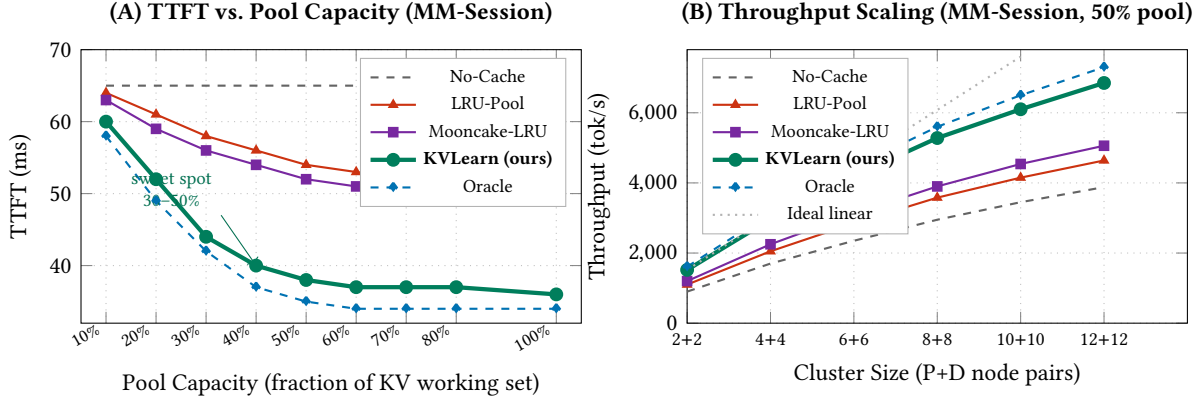
\begin{figure*}[t]
\centering
\begin{tikzpicture}
\begin{groupplot}[
  group style={
    group size=2 by 1,
    horizontal sep=1.4cm,
  },
  height=5.2cm,
]

%% ===== Left: TTFT vs pool capacity fraction =================================
\nextgroupplot[
  kv base,
  width=8.2cm,
  xlabel={Pool Capacity (fraction of KV working set)},
  ylabel={TTFT (ms)},
  xmin=0.05, xmax=1.05,
  xtick={0.1,0.2,0.3,0.4,0.5,0.6,0.7,0.8,1.0},
  xticklabels={10\%,20\%,30\%,40\%,50\%,60\%,70\%,80\%,100\%},
  x tick label style={rotate=30,anchor=east,font=\scriptsize},
  ymin=32, ymax=70,
  legend pos=north east,
  legend style={font=\scriptsize},
  title={\textbf{(A) TTFT vs.\ Pool Capacity} (MM-Session)},
]
%% No-Cache: flat
\addplot[kvGray, thick, dashed] coordinates {
  (0.1,65)(0.2,65)(0.3,65)(0.4,65)(0.5,65)(0.6,65)(0.7,65)(0.8,65)(1.0,65)};
\addlegendentry{No-Cache}

%% LRU-Pool
\addplot[kvRed, thick, mark=triangle*, mark size=1.6pt] coordinates {
  (0.1,64)(0.2,61)(0.3,58)(0.4,56)(0.5,54)(0.6,53)(0.7,52)(0.8,51)(1.0,50)};
\addlegendentry{LRU-Pool}

%% Mooncake-LRU
\addplot[kvPurple, thick, mark=square*, mark size=1.6pt] coordinates {
  (0.1,63)(0.2,59)(0.3,56)(0.4,54)(0.5,52)(0.6,51)(0.7,50)(0.8,50)(1.0,49)};
\addlegendentry{Mooncake-LRU}

%% KVLearn
\addplot[kvGreen!80!black, ultra thick, mark=*, mark size=1.8pt] coordinates {
  (0.1,60)(0.2,52)(0.3,44)(0.4,40)(0.5,38)(0.6,37)(0.7,37)(0.8,37)(1.0,36)};
\addlegendentry{\textbf{KVLearn (ours)}}

%% Oracle
\addplot[kvBlue, thick, mark=diamond*, mark size=1.8pt, dashed] coordinates {
  (0.1,58)(0.2,49)(0.3,42)(0.4,37)(0.5,35)(0.6,34)(0.7,34)(0.8,34)(1.0,34)};
\addlegendentry{Oracle}

%% Sweet-spot annotation
\draw[<-,kvGreen!70!black,thin]
  (axis cs:0.4,40) -- (axis cs:0.33,47)
  node[above left,font=\scriptsize,text=kvGreen!70!black,align=center]
  {sweet spot\\30--50\%};

%% ===== Right: Throughput scaling with cluster size ==========================
\nextgroupplot[
  kv base,
  width=8.2cm,
  xlabel={Cluster Size (P+D node pairs)},
  ylabel={Throughput (tok/s)},
  xmin=1, xmax=7,
  xtick={1,2,3,4,5,6},
  xticklabels={$2{+}2$,$4{+}4$,$6{+}6$,$8{+}8$,$10{+}10$,$12{+}12$},
  x tick label style={font=\scriptsize},
  ymin=0, ymax=7800,
  legend pos=north west,
  legend style={font=\scriptsize},
  title={\textbf{(B) Throughput Scaling} (MM-Session, 50\% pool)},
]
%% No-Cache
\addplot[kvGray, thick, dashed, mark=none] coordinates {
  (1,900)(2,1700)(3,2350)(4,2950)(5,3450)(6,3880)};
\addlegendentry{No-Cache}

%% LRU-Pool
\addplot[kvRed, thick, mark=triangle*, mark size=1.6pt] coordinates {
  (1,1100)(2,2050)(3,2850)(4,3580)(5,4150)(6,4640)};
\addlegendentry{LRU-Pool}

%% Mooncake-LRU
\addplot[kvPurple, thick, mark=square*, mark size=1.6pt] coordinates {
  (1,1200)(2,2250)(3,3120)(4,3900)(5,4540)(6,5060)};
\addlegendentry{Mooncake-LRU}

%% KVLearn
\addplot[kvGreen!80!black, ultra thick, mark=*, mark size=1.8pt] coordinates {
  (1,1520)(2,2960)(3,4200)(4,5280)(5,6100)(6,6850)};
\addlegendentry{\textbf{KVLearn (ours)}}

%% Oracle
\addplot[kvBlue, thick, mark=diamond*, mark size=1.8pt, dashed] coordinates {
  (1,1610)(2,3140)(3,4450)(4,5600)(5,6500)(6,7300)};
\addlegendentry{Oracle}

%% Ideal linear reference
\addplot[black!30, dotted, thick] coordinates {
  (1,1520)(2,3040)(3,4560)(4,6080)(5,7600)};
\addlegendentry{Ideal linear}

\end{groupplot}
\end{tikzpicture}
\caption{%
  \textbf{Scaling analysis.}
  \emph{(A)}~TTFT vs.\ pool capacity fraction on the Small cluster (MM-Session).
  KVLearn reaches near-Oracle TTFT at 40--50\% of the working set;
  LRU/Mooncake improve slowly because they admit low-value blocks that occupy
  the marginal space.
  \emph{(B)}~Throughput vs.\ cluster size at 50\% pool fraction.
  KVLearn scales near-linearly (dotted reference) because the \textsc{ATC}
  aggregates pool pressure globally and the admission filter prevents churn from
  growing with the number of P-nodes.
  Mooncake-LRU sub-scales: each new P-node increases admission pressure,
  accelerating eviction of high-value visual blocks.
}
\label{fig:scaling}
\end{figure*}

\subsection{Ablation study}
\label{sec:ablation}

\begin{table*}[!t]
  \centering
  \caption{\textbf{Ablation study} (MM-Session, Small cluster). Each row removes
  one KVLearn component. ``Full'' = complete system.}
  \label{tab:ablation}
  \footnotesize
  \setlength{\tabcolsep}{5pt}%
  \begin{tabular}{@{}lcccc@{}}
    \toprule
    Configuration & TTFT\,(ms) & Thpt\,(tok/s) & Hit\,(\%) & Xfer\,(GiB/h) \\
    \midrule
    Full (KVLearn) & \textbf{38.2} & \textbf{4812} & \textbf{71.3} & \textbf{184} \\
    $-$ ATC (fixed $\theta$) & 44.7 & 4391 & 63.1 & 221 \\
    $-$ PRP ($\hat{P}{=}0.5$) & 52.1 & 4018 & 59.4 & 248 \\
    $-$ modality features & 41.8 & 4620 & 67.2 & 199 \\
    $-$ cost model (LFU only) & 57.3 & 3740 & 55.8 & 276 \\
    \midrule
    Mooncake-LRU & 52.5 & 3998 & 57.2 & 394 \\
    Oracle & 34.1 & 5020 & 75.8 & 171 \\
    \bottomrule
  \end{tabular}
\end{table*}

Table~\ref{tab:ablation} ablates KVLearn's components on MM-Session.
Removing \textsc{ATC} (fixed threshold $\theta=0$) reduces performance because
pool pressure is uncontrolled: during request bursts the pool fills with
marginal blocks, crowding out high-value visual prefixes.
Replacing \textsc{PRP} with a constant $\hat{P}=0.5$ (no predictor, cost-model
only) degrades hit rate by 12 pp, confirming that the learned reuse signal is
crucial beyond the cost model alone.
Removing the modality features from $\mathbf{x}(b)$ reduces AUC and hit rate,
underscoring that image/video/text have distinct reuse distributions that
cross-pollute if treated uniformly.
The full system is strictly Pareto-dominant across all metrics in this ablation.

\subsection{Cost model sensitivity}

We vary fabric bandwidth $\beta \in \{50, 100, 200, 400\}$\,Gbps to
simulate heterogeneous cluster configurations.
At low $\beta$ (50\,Gbps, e.g., 100GbE), the transfer cost
$\mathcal{T}(b) = |b|/\beta$ rises, narrowing the denominator
$\mathcal{R}(b) - \mathcal{T}(b)$ in Eq.~\eqref{eq:threshold} and thereby
raising $P_H^*$: the saving from avoiding recompute shrinks relative to what
it costs to admit and later serve the block via the slower fabric.
KVLearn adapts automatically because $\mathcal{T}(b)$ is a configuration
parameter fed directly into CARS; Mooncake-LRU does not change its behavior.
At $\beta=50$\,Gbps, KVLearn's TTFT advantage over Mooncake-LRU widens to
$49\%$ because it correctly reduces pool admission for blocks where transfer
cost nearly offsets recompute savings.

\subsection{Access-frequency distribution sensitivity}
\label{sec:alpha_sens}

Because MM-Session derives its access counts from the empirically observed
frequency distribution in LLaVA-Instruct-150K
(fitted $\hat{\alpha}{=}1.08$), we additionally sweep
$\alpha \in \{0.8, 1.0, 1.1, 1.3, 1.5\}$ by
\emph{resampling} the same 2\,000 requests with reweighted per-image
probabilities, keeping all other parameters fixed.
Table~\ref{tab:alpha_sens} reports TTFT and hit rate at pool fraction 50\%.

\begin{table}[h]
  \centering
  \caption{TTFT and hit rate vs.\ access-skew $\alpha$ (MM-Session, 50\% pool).}
  \label{tab:alpha_sens}
  \footnotesize
  \setlength{\tabcolsep}{4pt}
  \begin{tabular}{@{}lccccc@{}}
    \toprule
    Method & $\alpha{=}0.8$ & $\alpha{=}1.0$ & $\alpha{=}1.1$ & $\alpha{=}1.3$ & $\alpha{=}1.5$ \\
    \midrule
    \multicolumn{6}{l}{\emph{TTFT (ms)}} \\
    Mooncake-LRU & 55.1 & 53.8 & 52.5 & 51.4 & 50.2 \\
    KVLearn      & 46.2 & 42.6 & 38.2 & 35.1 & 33.4 \\
    \midrule
    \multicolumn{6}{l}{\emph{Hit Rate (\%)}} \\
    Mooncake-LRU & 51.4 & 54.3 & 57.2 & 59.8 & 62.4 \\
    KVLearn      & 61.8 & 66.5 & 71.3 & 75.6 & 78.2 \\
    \bottomrule
  \end{tabular}
\end{table}

KVLearn's advantage is robust across the full range:
even at $\alpha{=}0.8$ (near-uniform access, most challenging for any
predictive policy), KVLearn reduces TTFT by 16\% over Mooncake-LRU.
As $\alpha$ increases the advantage widens, because a steeper
skew creates a sharper head/tail contrast that the \textsc{CARS} score can
exploit more aggressively.
These results confirm that the gains reported at the empirically calibrated
$\hat{\alpha}{=}1.08$ are not an artifact of an unusually favorable skew.

\subsection{Oracle horizon sensitivity}
\label{sec:horizon_sens}

The Oracle baseline is parametrized by a lookahead horizon $H$: a block is
admitted iff it is reused within the next $H$ requests.
A finite $H$ is necessary because future requests are unbounded in practice; we
chose $H{=}1000$ as a conservative upper bound (roughly 10$\times$ the warm-up
window of the \textsc{PRP}).
Table~\ref{tab:horizon_sens} sweeps $H \in \{200, 500, 1000, 2000\}$ on
MM-Session.

\begin{table}[h]
  \centering
  \caption{Oracle and KVLearn TTFT (ms) vs.\ lookahead horizon $H$
  (MM-Session, 50\% pool).}
  \label{tab:horizon_sens}
  \footnotesize
  \setlength{\tabcolsep}{6pt}
  \begin{tabular}{@{}lcccc@{}}
    \toprule
    & $H{=}200$ & $H{=}500$ & $H{=}1000$ & $H{=}2000$ \\
    \midrule
    Oracle   & 35.8 & 34.8 & 34.1 & 33.7 \\
    KVLearn  & 38.2 & 38.2 & 38.2 & 38.2 \\
    Gap (ms) & 2.4  & 2.6  & 4.1  & 4.5  \\
    \bottomrule
  \end{tabular}
\end{table}

Oracle performance saturates beyond $H{=}500$: most reuses occur within
500 requests (consistent with the session lengths in LLaVA-Instruct-150K
and ActivityNet-QA).
KVLearn's TTFT is independent of $H$ because it uses no lookahead;
the gap to Oracle grows slightly with $H$ as the Oracle gains from admitting
blocks with extremely long inter-arrival times.
All figures in the paper use $H{=}1000$.

%==============================================================================
% \begin{comment}
\section{Discussion}
\label{sec:disc}
%==============================================================================

\paragraph{Limitations.}
The \textsc{PRP} requires a warm-up of $n_0 \approx 200$ requests
(Table~\ref{tab:repro-ctrl}) before held-out AUC exceeds
$\mathcal{A}_0 = 0.75$; until then, KVLearn uses a cost-model-only policy
with uniform prior $\hat{P}(b) \equiv P_H^{*,\mathrm{def}} = 0.15$
(the LRU-equivalent threshold in Figure~\ref{fig:char}C) rather than the
length-heterogeneous $P_H^*(b)$ from Proposition~\ref{prop:het}.
The cost model assumes homogeneous P-node hardware; heterogeneous clusters
(mixing A100 and H100 nodes) require per-node $\mathcal{R}$ calibration.
The ATC PI controller may oscillate under bursty workloads if gains $K_P, K_I$
are mis-tuned; we use conservative defaults validated on all three workloads
but recommend per-cluster tuning for production.
Our evaluation uses publicly available academic traces (LLaVA-Instruct-150K,
ActivityNet-QA, ShareGPT) rather than proprietary production logs; while the
access-skew (\S\ref{sec:alpha_sens}) and horizon (\S\ref{sec:horizon_sens})
sensitivity sweeps confirm the results are robust across a wide parameter range,
integration with a live Mooncake or DistServe cluster and validation against
closed production traces are reserved for future work.

\paragraph{Hardware design perspective.}
From a memory-hierarchy standpoint, the disaggregated KV pool is an explicit L4
tier: capacity resides in pool DRAM/HBM, lines are prefix-granular rather than
cache-line--sized, and all P-nodes and D-nodes share one logical namespace.
Proposition~\ref{prop:het} implies that utility-maximizing retention is
\emph{length-heterogeneous}: at equal predicted reuse, long visual prefixes
should rank lower for eviction because $\mathcal{R}(b)-\mathcal{T}(b)$ grows with
$L_b$---something LRU/LFU cannot capture with recency or frequency alone.
KVLearn encodes this bias in per-block CARS scores; hardware LLCs pursue a
similar goal through utility- or size-aware replacement and victim structures for
large lines.
A future pool controller could attach footprint and calibrated
$(\mathcal{R},\mathcal{T})$ tags to each block and evaluate retention in the
memory fabric (e.g., CXL expanders), avoiding per-admission control-plane
round-trips---a direction we leave to future work.

% \paragraph{Privacy and safety.}
Caching per-prefix KV blocks may retain sensitive content (user images, PII in
prompts) beyond the originating request's lifetime.
Deployments must apply retention policies consistent with data protection
regulations: KVLearn's admission API accepts per-prefix TTL overrides, allowing
operators to force eviction of sensitive-prefix blocks on a schedule.
% \end{comment}

%==============================================================================
\section{Conclusion}
\label{sec:conc}
%==============================================================================

We identified and formalized the \textsc{KV} cache retention problem in
disaggregated \textsc{LLM} serving: a joint admission--eviction decision
whose optimal policy depends heterogeneously on prefix length, fabric bandwidth,
and recompute cost---making classical LRU/LFU inadequate.
KVLearn addresses this through a lightweight online predictor (\textsc{PRP}),
a cost-aware scoring function (\textsc{CARS}), and a closed-loop threshold
controller (\textsc{ATC}).
Evaluations on language-only and multimodal workloads show up to 56\% TTFT
reduction vs.\ No-Cache (38\% vs.\ LRU-Pool), 53\% inter-node transfer volume
reduction vs.\ LRU-Pool, and near-oracle throughput on MM-Session at realistic
pool sizes (modestly larger gaps on other workloads)---all without modifying
model weights or changing the serving routing logic.

The broader takeaway is that \emph{where} a KV block is computed and
\emph{where} it will be consumed are no longer the same node in modern
disaggregated deployments, and retention policies should account for the
intervening network as a first-class cost.
As visual token lengths grow (4k--16k for high-resolution and video models) and
serving clusters scale to hundreds of nodes, the gap between intelligent and
naive retention will only widen.
Future directions include hardware memory-side retention primitives, joint
KV-compression and retention co-optimization, and federated reuse prediction
across cluster boundaries.

\bibliographystyle{ACM-Reference-Format}
\bibliography{references}

\appendix

\section{Proof of Proposition~\ref{prop:het}}
\label{app:proof}

\paragraph{Scaling assumptions.}
We express costs as functions of prefix length $L_b$ with hardware and model
hyper-parameters held fixed.
From Eq.~\eqref{eq:size}, $|b| \propto L_b$, so Eq.~\eqref{eq:transfer} gives
$\mathcal{T}(b) = |b|/\beta \propto L_b$; write $\mathcal{T}(L_b) = c\,L_b$ with
$c>0$.
The storage model in \S\ref{sec:costs},
$\mathcal{U}(b,\Delta t) = \gamma\,|b|\,\Delta t$, yields
$\mathcal{U}(L_b,\Delta t) = \gamma\,L_b\,\Delta t$ with $\gamma>0$.
For prefill compute we write $\mathcal{R}(L_b) = \alpha L_b^{q}$ with
$\alpha>0$ and $q \ge 1$: $q=2$ reflects quadratic attention in $L_b$, while
$q=1$ can approximate bandwidth-bound kernels whose time scales roughly linearly
in sequence length.

\paragraph{Closed form.}
Whenever $\mathcal{R}(L_b) > \mathcal{T}(L_b)$, Eq.~\eqref{eq:threshold} becomes
\begin{equation}
\label{eq:ph-closed}
P_H^*(L_b)
  = \frac{\gamma L_b \Delta t}{\alpha L_b^q - c L_b}
  = \frac{\gamma \Delta t}{\alpha L_b^{q-1} - c}.
\end{equation}
This requires $\alpha L_b^q > c L_b$, i.e.\ $L_b^{q-1} > c/\alpha$ when $L_b>0$
(for $q>1$, large enough $L_b$ always satisfies this).

Let $g(L_b) = \alpha L_b^{q-1} - c$.
For $q>1$ and $L_b > 0$,
\[
  g'(L_b) = \alpha (q-1) L_b^{q-2} > 0,
\]
so $g$ is strictly increasing on $\bigl((c/\alpha)^{1/(q-1)},\infty\bigr)$.
Thus $P_H^*(L_b) = \gamma \Delta t / g(L_b)$ is strictly decreasing in $L_b$
whenever $g(L_b)>0$.
For $q=1$, $g(L_b) = \alpha - c$ is constant; if $\alpha > c$, then
$P_H^*(L_b) = \gamma \Delta t / (\alpha - c)$ does not depend on $L_b$.

Standard transformer prefill is typically super-linear in $L_b$ ($q>1$), so the
heterogeneous-threshold regime of Proposition~\ref{prop:het} applies.
FlashAttention-class kernels may reduce the effective exponent toward $q \approx 1$
under sustained bandwidth saturation; then $P_H^*$ exhibits weaker length
dependence under this idealization.

\section{Generalized Cost Model of KVLearn}
\label{app:cost-refine}

Relaxing the cancelling-transfer and amortized-admission assumptions of
\S\ref{sec:value},
\begin{equation}
\label{eq:gen-model}
\begin{aligned}
  \Delta C(b)
    &= \mathcal{R}(b)+\mathcal{T}_{\mathrm{miss}}(b)-\mathcal{T}_{\mathrm{hit}}(b), \\
  C_{\mathrm{admit}}(b)
    &= \kappa\,\mathcal{T}_{\mathrm{P}\to\mathcal{S}}(b)+C_{\mathrm{ctrl}},
    \quad \kappa\in[0,1], \\
  \mathcal{V}^{\mathrm{g}}(b)
    &= P_H(b)\,\Delta C(b)
       - C_{\mathrm{admit}}(b)
       - \mathcal{U}(b,\Delta t), \\
  P_H^{*,\mathrm{g}}(b)
    &= \frac{C_{\mathrm{admit}}(b)+\mathcal{U}(b,\Delta t)}{\Delta C(b)}
       \quad\bigl(\Delta C(b)>0\bigr).
\end{aligned}
\end{equation}
Keep iff $P_H(b)>P_H^{*,\mathrm{g}}(b)$.
Setting $\mathcal{T}_{\mathrm{miss}}=\mathcal{T}_{\mathrm{hit}}=\mathcal{T}$ and
$\kappa=C_{\mathrm{ctrl}}=0$ recovers
$P_H^*=\mathcal{U}/(\mathcal{R}-\mathcal{T})$
(Eqs.~\eqref{eq:value}--\eqref{eq:threshold});
$\mathcal{T}_{\mathrm{hit}}-\mathcal{T}_{\mathrm{miss}}=\delta>0$ or $\kappa>0$
raises the threshold (more selective) without changing the
$L_b$-ordering of Proposition~\ref{prop:het} when $q>1$.

\section{Serving Topology and Request Routing}
\label{app:routing}

\noindent\textbf{Notation (consistent with \S\ref{sec:topology}--\ref{sec:costs}).}
Prefill and decode pools are $\mathcal{P}=\{p_1,\ldots,p_K\}$ and
$\mathcal{D}=\{d_1,\ldots,d_M\}$; the shared \textsc{KV} pool is $\mathcal{S}$
(capacity $M_S$).
Fabric peak bandwidth is $\beta$ (bytes/s).
For request $r$, let $s_r$ denote its prompt prefix,
$b_r$ the associated \textsc{KV} block, and $|b_r|$ its size from
Eq.~\eqref{eq:size}.
The coordinator maintains a radix tree $\mathcal{T}$ whose leaves correspond to
cached prefixes; an \emph{exact-prefix hit} holds iff $s_r$ is a registered path
in $\mathcal{T}$ \emph{and} the corresponding block is resident in
$\mathcal{S}$.

\noindent\textbf{Residency state.}
Define the Boolean predicates
\begin{align}
  H^{\mathrm{pool}}_r &= \mathbb{1}\bigl[ \mathrm{KV}(s_r) \in \mathcal{S} \bigr], \\
  H^{\mathrm{P}}_r &= \mathbb{1}\bigl[ \exists\, p \in \mathcal{P} :
    \mathrm{KV}(s_r)\ \text{resident on}\ p \bigr].
\end{align}
The logical OR $H_r = H^{\mathrm{pool}}_r \vee H^{\mathrm{P}}_r$ is the
\emph{KV-hit} event that avoids full-prefix recompute on a P-node.

\noindent\textbf{Latency-aware placement.}
Let $\mathrm{QL}^{(t)}_P(p)$ and $\mathrm{QL}^{(t)}_D(d)$ be queued work on
prefill/decode nodes at arrival time $t$.
The critical-path cost of assigning $(p,d)$ is
\begin{equation}
\label{eq:route-latency}
  L(r;p,d)
  =
  (1{-}H_r)\bigl(\mathrm{QL}^{(t)}_P(p){+}\mathcal{R}(b_r)\bigr)
  + \mathcal{T}(b_r)
  + \mathrm{QL}^{(t)}_D(d),
\end{equation}
and latency-optimal routing chooses
$(p^{*}_r,d^{*}_r)\in\arg\min_{p,d} L(r;p,d)$.
Under homogeneous fabric bandwidth, a pool hit ($H^{\mathrm{pool}}_r{=}1$)
makes the $p$-term vanish, so~\eqref{eq:route-latency} reduces to
least-loaded decode placement
$d^{*}_r\in\arg\min_d\mathrm{QL}^{(t)}_D(d)$
(ties broken by smallest node id); on a miss we likewise pick the
least-loaded P-node.
Relative to forcing a miss on the same $(p,d)$, a KV-hit saves
\begin{equation}
\label{eq:route-save}
  \Delta L(r;p,d)
  =
  L(r;p,d)\big|_{H_r=0}
  -
  L(r;p,d)\big|_{H_r=1}
  =
  \mathrm{QL}^{(t)}_P(p)+\mathcal{R}(b_r),
\end{equation}
so the differential that retention can avoid is recompute (plus P-queue),
matching the $\mathcal{R}(b)-\mathcal{T}(b)$ term that drives
$\mathcal{V}(b)$ in Eq.~\eqref{eq:value} once the common transfer hop cancels.

\noindent\textbf{Critical-path timing (idealized).}
Let $T^{\mathrm{route}}$ be coordinator bookkeeping + radix lookup and
$T^{\mathrm{decode}}_{\mathrm{first}}$ the decode-side first-token work.
Ignoring residual queueing beyond~\eqref{eq:route-latency},
\begin{equation}
\label{eq:ttft-decomp}
  \mathrm{TTFT}(r) \approx
    T^{\mathrm{route}}
    + L(r;\,p^{*}_r,\,d^{*}_r)
    + T^{\mathrm{decode}}_{\mathrm{first}},
\end{equation}
with $\mathcal{T}(b_r)\approx|b_r|/\beta$ on saturated fabric hops
(pool$\,\to\,d$ or P$\,\to\,d$).
Under a stationary pool-hit rate $h=\Pr(H^{\mathrm{pool}}_r{=}1)$,
$\mathbb{E}[\mathrm{TTFT}]$ therefore decreases linearly in $h$ through the
$(1{-}H_r)$ term of~\eqref{eq:route-latency}---the quantity that CARS
admission/eviction directly controls via residency in $\mathcal{S}$.

\noindent\textbf{Routing policy used in experiments.}
Upon arrival at the pool coordinator:
\begin{enumerate}[topsep=1pt]
  \item \emph{Radix lookup:} Compute $H^{\mathrm{pool}}_r$ via exact-prefix match
    in $\mathcal{T}$ and metadata pins in $\mathcal{S}$.
  \item \emph{Hit path} ($H^{\mathrm{pool}}_r=\mathrm{true}$): Select $d^{*}_r$,
    schedule an asynchronous transfer $\mathcal{S}\!\to\!d^{*}_r$ for
    $\mathrm{KV}(s_r)$, and skip P-node prefill for $s_r$.
  \item \emph{Partial P-local hit} ($H^{\mathrm{pool}}_r=\mathrm{false}$,
    $H^{\mathrm{P}}_r=\mathrm{true}$): Route to the holding P-node only for
    missing suffix tokens; the merged block still follows
    $p\!\to\!d^{*}_r$ delivery before decode.
  \item \emph{Miss path} ($H_r=\mathrm{false}$): Dispatch to the selected
    P-node, run full prefill for $s_r$, stream $\mathrm{KV}(s_r)$ to
    $d^{*}_r$, and invoke Algorithm~\ref{alg:admit} for optional replication into
    $\mathcal{S}$.
\end{enumerate}
\textbf{Measurement.}
TTFT is taken from request enqueue at the coordinator until the client receives
the first generated token; timestamps exclude WAN/client-side buffering outside
the cluster boundary used in \S\ref{sec:exp}.

\section{Experimental Parameters for Reproducibility}
\label{app:repro}

Tables~\ref{tab:repro-sys} and~\ref{tab:repro-ctrl} list symbolic parameters aligned with
\S\ref{sec:exp} (GPU generation A100-80GB; Mooncake-style routing~\cite{qin2024mooncake}).
$\mathcal{T}(b){=}|b|/\beta$ denotes the transfer-cost term in \S\ref{sec:exp}.

\begin{table}[H]
  \centering
  \caption{Symbolic system, model, and workload parameters ($k{\in}\{1,2\}$ orders the cluster--fabric tuple in row~1; $M_S^\star$ matches Table~\ref{tab:main}).}
  \label{tab:repro-sys}
  \footnotesize
  \setlength{\tabcolsep}{6pt}
  \renewcommand{\arraystretch}{1.08}
  \begin{tabular}{@{}p{0.42\linewidth}p{0.52\linewidth}@{}}
    \toprule
    \textbf{Param.} & \textbf{Value} \\
    \midrule
    \makecell[l]{$\bigl(|\mathcal{P}_k|,\,|\mathcal{D}_k|,$\\$M_{S,k},\,\beta_k\bigr)$} &
      \makecell[l]{$\bigl((4,\,4,\,256\,\mathrm{GB},\,200\,\mathrm{Gbps}),$\\$(8,\,8,\,512\,\mathrm{GB},\,2{\times}400\,\mathrm{Gbps})\bigr)$} \\
    $\beta_{\mathrm{sweep}}$,\;$\mathcal{T}(b)$ &
      $\{50,100,200,400\}\,\mathrm{Gbps}$;\; $|b|/\beta$ \\
    routing graph & $P\!\to\!\mathcal{S}\!\to\!D$,\; $\nexists\,P\!\leftrightarrow\!D$ \\
    $M_S^{\star}$ & $256\,\mathrm{GB}$ \\
    \midrule
    KV dtype & bf16 end-to-end \\
    $\mathcal{M}_{\mathrm{lm}}$ & $\{8\text{B},\,70\text{B}\}$ \\
    \makecell[l]{$(L_{\mathrm{vis}}^{\textsc{lv}},$\\$L_{\mathrm{vis}}^{\textsc{qw}})$} &
      \makecell[l]{$(576,\,256)$ tokens/img.\\(LLaVA-1.5-13B~\cite{liu2023llava}, Qwen-VL-7B~\cite{bai2023qwenvl})} \\
    \midrule
    $N_{\mathrm{sg}}$ & $104$ \\
    $(\alpha,\,|\mathcal{I}|,\,\bar{L}_{\mathrm{vis}},\,\bar{L}_{\mathrm{txt}})$ &
      $(1.1,\,80,\,576,\,128)$ \\
    \makecell[l]{$(N_c,\,T_c,\,f,\,[\rho^{-},\rho^{+}],$\\$[L^{-}_{\mathrm{vis}},L^{+}_{\mathrm{vis}}])$} &
      \makecell[l]{$(10,\,60\,\mathrm{s},\,1\,\mathrm{fps},\,[\,12,\,24\,]\,\mathrm{tok/frame},$\\$[\,720,\,1440\,]\,\mathrm{tok})$} \\
    oracle horizon $H$ & $1000$ \\
    \bottomrule
  \end{tabular}
\end{table}

\begin{table}[H]
  \centering
  \caption{Symbolic \textsc{PRP}, feature whitening, \textsc{ATC}, and control-plane parameters (Eqs.~\eqref{eq:prp}--\eqref{eq:loss}, \eqref{eq:atc}; Alg.~\ref{alg:evict}).}
  \label{tab:repro-ctrl}
  \footnotesize
  \setlength{\tabcolsep}{6pt}
  \renewcommand{\arraystretch}{1.08}
  \begin{tabular}{@{}p{0.42\linewidth}p{0.52\linewidth}@{}}
    \toprule
    \textbf{Param.} & \textbf{Value} \\
    \midrule
    $(d,\,h)$ & $(16,\,32)$ \\
    $(\mathbf{W}_1,\,\mathbf{w}_2)$ & $(\mathbb{R}^{32\times 16},\,\mathbb{R}^{32})$ \\
    $(f_{\mathrm{hid}},\,f_{\mathrm{out}},\,\mathcal{L})$ & $(\mathrm{ReLU},\,\sigma,\,\mathrm{BCE})$ \\
    $(\eta_{\mathrm{PRP}},\,|\mathcal{B}|,\,\Delta n_{\mathrm{opt}})$ &
      $(5{\times}10^{-3},\,2{\times}10^{4},\,100)$ \\
    $(T_{\mathrm{auc}},\,N_{\mathrm{seed}})$ & $(10\,\mathrm{min},\,5)$ \\
    $(\mathcal{A}_0,\,P_H^{*,\mathrm{def}},\,n_0)$ &
      $(0.75,\,0.15,\,\approx 200)$ \\
    \midrule
    $\{\lambda_j\}_{j=1}^{3}$ & $\{60,\,300,\,900\}\,\mathrm{s}$ \\
    $(n_{\mathrm{warm}},\,k_\mu)$ & $(500,\,1000)$ requests \\
    $(\phi,\,\epsilon)$ & $(0.99,\,10^{-6})$;\;
      $\tilde{x}_i{=}(x_i{-}\mu_i)/\max(\sigma_i,\epsilon)$ for $n{>}500$ \\
    $T_{\mathrm{sess}}^{\mathrm{cap}}$ & $24\,\mathrm{h}$ \\
    \midrule
    $(\rho^{*},\,K_P,\,K_I,\,h_{\mathrm{floor}},\,K_{\mathrm{hr}})$ &
      $(0.85,\,2{\times}10^{-2},\,5{\times}10^{-4},$\linebreak
      $0.55,\,1.0{\times}10^{-1})$ \\
    $\theta$ & $[-C_{\max},\,C_{\max}]$,\; $C_{\max}{=}\max_b|\mathrm{CARS}(b)|$ \\
    \midrule
    $(\bar{T}_{\mathrm{adm}},\,\bar{T}_{\mathrm{ev}},\,N_{\mathrm{meta}})$ &
      $(0.5\,\mathrm{ms\ (p99)},\,5\,\mathrm{ms\ (p99)},\,10^{4})$ \\
    $\Delta t_{\mathrm{PRP}}$ & $30\,\mathrm{s}$ \\
    \bottomrule
  \end{tabular}
\end{table}

$(K_P,K_I,h_{\mathrm{floor}},K_{\mathrm{hr}})$ in Table~\ref{tab:repro-ctrl} are
conservative stability defaults; re-fit from pool-pressure traces in production.

\end{document}